\documentclass[twoside,11pt]{article}
\usepackage{blindtext}
\usepackage[table,dvipsnames]{xcolor}
\usepackage[preprint]{jmlr2e}

\usepackage{jmlr2e}
\usepackage{amsmath}
\numberwithin{equation}{section}
\usepackage{tcolorbox}
\usepackage{hyperref}
\usepackage{subfigure}
\usepackage{booktabs}    
\usepackage{array}       
\usepackage{makecell}    
\usepackage{paralist}    
\usepackage{paralist}
\usepackage{tikz}
\usepackage{multirow}
\usepackage{algorithm, algpseudocode,}

\newcommand{\best}[1]{\cellcolor{RoyalBlue!25}\textbf{#1}}   %
\newcommand{\second}[1]{\cellcolor{RoyalBlue!15}#1}          %
\newcommand{\third}[1]{\cellcolor{RoyalBlue!7}#1}           
\renewcommand{\arraystretch}{1.2}  

\newcommand{\I}{\mathrm{I}}

\newcommand{\KL}{\mathrm{KL}}
\newcommand{\IB}{\mathrm{IB}}
\newcommand{\E}{\mathbb{E}}
\newcommand{\mix}{\mathrm{mix}}

\usepackage{lastpage}

\ShortHeadings{Observations Meet Actions}{Gu and Cao}
\firstpageno{1}

\begin{document}

\title{Observations Meet Actions: Learning Control-Sufficient Representations  for Robust Policy Generalization}

\author{\name Yuliang Gu \email yuliang3@illinois.edu \\
       \addr Department of Mechanical Science and Engineering\\
       University of Illinois Urbana-Champaign\\
       Urbana, IL 61801, USA
       \AND
       \name Hongpeng Cao \email cao.hongpeng@tum.de \\
       \addr School of Engineering and Design\\
       Technical University of Munich\\
       Garching, Munich 85748, Germany
       \AND
       \name Marco Caccamo \email mcaccamo@tum.de \\
       \addr School of Engineering and Design\\
       Technical University of Munich\\
       Garching, Munich 85748, Germany
       \AND
       \name Naira Hovakimyan \email nhovakim@illinois.edu \\
       \addr Department of Mechanical Science and Engineering\\
       University of Illinois Urbana-Champaign\\
       Urbana, IL 61801, USA}
       
\editor{My editor}
\maketitle

\begin{abstract}
Capturing latent variations (\textit{``contexts''}) is key to deploying reinforcement-learning (RL) agents beyond their training regime. We recast context-based RL as a dual inference–control problem and formally characterize two properties and their hierarchy: \emph{observation sufficiency} (preserving all predictive information) and \emph{control sufficiency} (retaining decision-making relevant information). Exploiting this dichotomy, we derive a contextual evidence lower bound(ELBO)-style objective that cleanly separates representation learning from policy learning and optimizes it with \textbf{Bottlenecked Contextual Policy Optimization (BCPO)}, an algorithm that places a variational information-bottleneck encoder in front of any off-policy policy learner. On standard continuous-control benchmarks with shifting physical parameters, BCPO matches or surpasses other baselines while using fewer samples and retaining performance far outside the training regime. The framework unifies theory, diagnostics, and practice for context-based RL.
\end{abstract}


\section{Introduction}\label{sec:intro}
Reinforcement Learning (RL) has recently demonstrated (super-)human proficiency from mastering complex games~\citep{silver2017mastering} to executing dexterous robotic manipulation \citep{andrychowicz2020learning}. These achievements, however, were accomplished in controlled environments whose dynamics and objectives remain fixed during training and evaluation. Real-world applications rarely enjoy such static conditions. For example, robots face wear-and-tear, payload changes, or drifting sensor calibrations; autonomous vehicles encounter new road conditions and evolving traffic patterns; and assistive agents must adapt to users’ changing preferences~\citep{mao2023sℒ1}.

Practical deployments must therefore cope with \emph{latent contexts}, unobserved variables that can change from episode to episode by altering either the dynamics (e.g., a mobile robot carrying an unknown load), the reward structure (e.g., a newly specified goal), or the observation process (e.g., temporary sensor noise). When the true context drifts beyond the training distribution, monolithic ``single-context'' policies can suffer abrupt performance degradation~\citep {zhang2018study}.

A widely adopted approach is domain randomization and its extensions (e.g., curriculum-based sampling~\citep{eimer2021self}). By exposing the agent to a broad spectrum of environments, domain randomization encourages robustness to nuisance factors~\citep{tobin2017domain}. Despite its simplicity, this method demands detailed prior knowledge of all task-relevant variations and often requires virtually unlimited simulation budgets. Moreover, its black-box nature leaves the deployed agent with no explicit understanding of the specific context it encounters at test time.

A more transparent alternative casts the challenge as a \emph{dual} task of (i) \textbf{latent context inference} and (ii) \textbf{conditional control}~\citep{oord2018representation, rakelly2019efficient, lee2020context, benjaminscontextualize}. A dedicated \emph{context encoder} summarizes past interactions into a compact latent code, which the policy then conditions on to select actions. This ``grey‑box'' architecture retains interpretability while allowing each module to exploit specialized advances on each end: powerful representation learners such as variational and contrastive methods~\citep{kingma2013auto} for inference, and proven reinforcement learning algorithms such as Soft Actor‑Critic (SAC)~\citep{pmlr-v80-haarnoja18b} for control. The resulting \emph{modular} pipeline combines expressive context representations with sample‑efficient policy optimization, yielding agents that generalize more reliably across different scenarios.

Despite its encouraging empirical gains, the \emph{information structure} underlying such dual architecture in reinforcement learning is still poorly understood. In particular, several foundational questions remain obscure, or their answers are dispersed across disparate research communities:
\begin{itemize}
    \item \textbf{Inference Side.} \emph{What is the minimum amount of information an encoder must keep to ``name'' the hidden context variable with high confidence?}
    \item \textbf{Control Side.} \emph{Which parts of the information does a policy need to ``act'' optimally? Is full reconstruction of the state(-action) value landscape necessary?}
    \item \textbf{Coupling.}  \emph{How are the information demands of inference and control intrinsically linked, and can this link be exploited to build sample-efficient, robust algorithms?}
\end{itemize}

\medskip\noindent
Answering these questions above sharpens the theoretical picture of context-based RL and, crucially, shows \emph{how the representation learning(inference) must interact with the policy optimization(control)}. This paper provides those answers through three contributions:
\begin{enumerate}
    \item \textbf{Unified theoretical framework.} We give precise definitions of \emph{observation sufficiency} (for inference) and \emph{control sufficiency} (for decision-making), establish their hierarchy, and prove when and how one implies
    the other.
    
    \item \textbf{Principled Optimization Objective.} Embedding these definitions in a variational RL view, we
    derive an \emph{evidence lower bound} (ELBO) that cleanly separates representation refinement from policy improvement.  
    A novel \emph{information residual} decomposes into intrinsic
    (processing and encoding) and extrinsic (replay) information gaps, giving a transparent optimization pipeline.
    
    \item \textbf{Algorithmic Realization and Validation.} We instantiate the theory in \emph{Bottlenecked Contextual
    Policy Optimization} (\textsc{BCPO}), which allows plug-and-play of any off-the-shelf max-entropy RL methods. Experiments on continuous-control benchmarks show superior sample efficiency and robustness to context variations.
\end{enumerate}
\medskip
Together, these contributions offer an end-to-end route from
information-theoretic definitions to a robust, sample-efficient algorithm for context-based RL. The remainder of this paper is organized as follows: Section~\ref{sec:prelim} reviews information-theoretic and RL preliminaries; Sections~\ref{sec:problem_formulation} \&~\ref{sec:suff all} develop the sufficiency hierarchy and the information-residual decomposition;  Sections~\ref{sec:minimize info_residual} \&~\ref{sec:bcpo} present residual minimization and the \textbf{BCPO} algorithm; Section~\ref{sec:expr} reports empirical results; Section~\ref{sec:related} situates our work within the literature; and Section~\ref{sec:future} concludes with limitations and directions for future research.

\section{Preliminaries}
\label{sec:prelim}
\subsection{Reinforcement Learning}
Consider a Markov decision process (MDP)~\cite{sutton1998reinforcement} 
\(
  (\mathcal{S},\,\mathcal{A},\,\mathcal T,\,r, \gamma),
\)
where \(\mathcal{S}\) is the state space, \(\mathcal{A}\) the action space, $\mathcal T$ the stochastic transition kernel with probability distribution $p(s'|s,a)$, and \(r(s,a)\) the reward function.  For notational convenience, we set the discount factor \(\gamma=1\). The reinforcement learning (RL) objective is to find a (stochastic) policy \(\pi(a|s)\) that maximizes the expected return
\begin{equation}\label{eq:rl obj}
  J(\pi)
  = \mathbb{E}_{p(\tau)}
    \Bigl[R(\tau)\Bigr], \quad R(\tau) := \sum_{t=1}^{T}r(s_{t},a_{t})
\end{equation}
under the trajectory distribution
\[
    p(\tau) = p(s_1, a_1, \dots, s_T, a_T) = p(s_1) \prod_{t=1}^T p(s_{t+1}|s_t,a_t)\; \pi(a_t|s_t).
\]

\paragraph{RL as Probabilistic Inference.}  
Following the control-as-inference framework~\citep{ziebart2010modeling,levine2018reinforcement}, we attach a binary optimality variable $o_t$ to each $(s_t,a_t)$ with
$p(o_t=1\mid s_t,a_t) \propto \exp\{r(s_t,a_t)\}$.  The joint
density of a trajectory $\tau$ and the 
optimal state-action event
$o_{1:T}=1$ becomes
\[
  p(\tau,o_{1:T}=1) \propto p(\tau)\exp\{R(\tau)\},
\]
where $p(\tau)$ is the trajectory distribution induced by $\pi$ and
$R(\tau)=\sum_{t}r(s_t,a_t)$. Let $q_{\pi}(\tau)$ be the \emph{controlled} trajectory distribution induced by a policy
$\pi$.  Minimizing
$\mathrm{KL}\!\bigl(q_{\pi}(\tau)\,\Vert\,p(\tau\mid o_{1:T}=1)\bigr)$
is equivalent to \emph{maximizing} the evidence lower bound (ELBO)~\citep{levine2018reinforcement}:
\[
  \max_{\pi}\;\; \mathbb{E}_{q_{\pi}}
  \Bigl[R(\tau) - \log q_{\pi}(\tau)\Bigr].
\]
Because $-\log q_{\pi}(\tau) = \sum_{t}\!-\!\log\pi(a_t\!\mid\!s_t)$ up
to constants, and writing $\mathcal{H}\bigl(\pi(\!\cdot\!\mid s_t)\bigr)$
for the action entropy, we obtain the \emph{maximum-entropy RL (MaxEnt RL)} objective
\begin{equation}
  J
  \;=\;
  \mathbb{E}_{q_{\pi}}
  \Bigl[
    \sum_{t=1}^{T}
      \bigl(
        r(s_t,a_t)
        + \mathcal{H}\!\bigl(\pi(\!\cdot\!\mid s_t)\bigr)
      \bigr)
  \Bigr].
  \label{eq:maxent_rl}
\end{equation}

\subsection{Information–theoretic preliminaries}
\label{subsec:it_prelim}

Let $(X,Y)\sim p(x,y)$ be a pair of random variables with marginals
$p(x)$ and $p(y)$.

\paragraph{Entropy and mutual information.}
The (Shannon) entropy of $X$ is
\[
  H(X)\;=\;-\E_{p(x)}\!\bigl[\log p(x)\bigr].
\]
The \emph{mutual information} (MI) between $X$ and $Y$ measures how much
knowing one variable reduces uncertainty about the other~\citep{cover1999elements}:
\[
  I(X;Y)
  \;=\;
  \E_{p(x,y)}
    \!\Bigl[
      \log \tfrac{p(x,y)}{p(x)\,p(y)}
    \Bigr]
  \;=\;
  H(X)-H(X\!\mid Y)
  \;=\;
  H(Y)-H(Y\!\mid X),
\]
i.e.\ the Kullback–Leibler divergence between the joint distribution and
the product of the marginals:
\(
  I(X;Y)=\KL \bigl(p(x,y)\,\|\,p(x)p(y)\bigr).
\)

\paragraph{Data–processing inequality (DPI).}
If $X\!\to\!Z\!\to\!Y$ forms a Markov chain, then any \emph{post-processing}
cannot increase information:
\[
  I(X;Y)\;\le\;I(X;Z).
\]
Equality holds \emph{iff} $X$ and $Y$ are conditionally independent given
$Z$, i.e.\ $X\!\perp\!\!\!\perp Y\mid Z$.

\paragraph{Information Bottleneck (IB).}
Given an \emph{input} $X$ and a \emph{relevant target} $Y$, the
Information Bottleneck principle
\citep{tishby2000information,alemi2016deep} seeks a stochastic encoder
$q_\phi(z| x)$ that produces a compressed latent representation
$Z\!=\!Z_\phi(X)$ while retaining maximal predictive power about $Y$.
The canonical Lagrangian form is
\[
  \min_{q_\phi(z\mid x)}
  \;
  \beta\;I(X;Z)\;-\; I(Z;Y),
\]
where $I_\phi(\cdot;\cdot)$ denotes MI under the joint distribution
induced by $q_\phi$, and the non-negative hyper-parameter
$\beta$ trades \emph{compression} ($I(X;Z)$) against \emph{relevance}
($I(Z;Y)$).

\section{Contextual Reinforcement Learning}
\label{sec:problem_formulation}
We consider \textbf{contextual reinforcement learning (cRL)}, where the
environment’s dynamics are modulated at the episode level by an
\emph{unobserved} latent variable—called the \emph{context}—that is
sampled once at the beginning of each episode and remains fixed
thereafter. This formalism captures many practical scenarios, including domain randomization for sim-to-real transfer, context-augmented policies, and goal-conditioned control.

\paragraph{Contextual MDP.}
Formally, every context
\(c\in\mathcal C\) instantiates a Markov decision process
\[
  \mathcal M(c)
  \;:=\;
  \bigl(\mathcal S,\mathcal A,\mathcal T_{c},r_c\bigr),
\]
where \(\mathcal T_{c}(s'\!\mid s,a)\) is the state–transition kernel and \(r_c(s,a)\) is the associated reward function. We assume that both the dynamics and the reward share the same context parameterization.

At the start of an episode, the environment draws
\(c\sim p(c)\)%
\footnote{Unless stated otherwise, \(p(c)\) is taken to be uniform over $\mathcal{C}$.}
and the agent, \emph{without observing \(c\)}, interacts for \(T\) steps,
producing a trajectory
\(
  \tau=(s_{1:T},a_{1:T})
\)
drawn from the joint distribution \(p(c,\tau)\).

\paragraph{Objective.}
A \emph{context-conditioned policy}
\(
  \pi(a\mid s,c)
\)
achieves expected return
\begin{equation}
\label{eq:crl_objective}
  J(\pi)
  \;=\;
  \mathbb E_{p(c,\tau)}\!
    \Bigl[
      R(\tau)
    \Bigr],
  \qquad
  R(\tau)=\sum_{t=1}^{T} r_c\bigl(s_{t},a_{t}\bigr).
\end{equation}
If \(c\) were \emph{observed}, maximizing~\eqref{eq:crl_objective}
reduces to standard RL in the augmented state
\(\hat s:=(s,c)\).  Because \(c\) is hidden and stochastic, the agent
must \emph{infer} it from interaction before acting.

\paragraph{Motivation.}
Hidden contexts create two information demands:
\begin{enumerate}
  \item \textbf{Observation sufficiency.}
        \emph{How much information must the agent gather to identify the
        context?}
  \item \textbf{Control sufficiency.}
        \emph{How much information is required to act
        near-optimally without ever revealing the true context?}
\end{enumerate}


Generally speaking, the gap between the two is precisely the information about \(c\) that is \emph{irrelevant} for decision-making. Empirically, many RL methods exploit this idea implicitly: contrastive
predictive coding~\citep{oord2018representation}, goal-conditioned
latent models~\citep{zeng2023goal}, meta-RL context
inference~\citep{rakelly2019efficient}, domain-randomization with invariant features~\citep{higgins2017darla},
and adversarial domain generalization~\citep{hansen2021generalization}.
Yet these methods rely on heuristics; a \emph{principled framework} that quantifies the trade-off and guides algorithm design remains absent. The remainder of this paper supplies such a framework, giving information-theoretic definitions of both sufficiency notions, revealing their fundamental relationship, and proposing a practical algorithm.

\section{Observation \& Control Sufficiency}\label{sec:suff all}
\subsection{Observation (representation) Sufficiency}\label{sec:obs suff}
During an episode, the agent observes a \emph{contextual} interaction history
\(
  \tau=(s_{1:T},a_{1:T})\sim p(\tau|c),
\)
but never the latent context \(C\) itself.  We cast the ``observation (representation)
problem'' as an \emph{inference} task: an
encoder
\(
      q_{\phi}(z\mid \tau)
\)
, parameterized by $\phi$, compresses the trajectory into a latent code \(Z\), inducing the
data-processing chain  
\(
      C \;\to\; \tau \;\to\; Z .
\)

Intuitively, the encoder compresses the information-rich trajectory \(\tau\) into a summary \(Z\) whose role is to extract exactly the bits of $\tau$ that reveal which hidden \(C\) generated the data, so
that $Z$ can serve as a surrogate for the unobserved context. The processing order $C \to \tau \to Z$ forms a Markov chain, which imposes two crucial physical constraints: \textbf{(i)} All information about $C$ available to the encoder must flow through the trajectory, and \textbf{(ii)} $Z$ can never be \emph{more} informative about the context than the trajectory itself.

Processing an \emph{entire} trajectory is often unnecessary and, for
high-dimensional control tasks, computationally prohibitive.  In many systems, the context is revealed by only a few early transitions, e.g., a cart–pole whose pole mass can be inferred from the first swings. We
therefore introduce a length-\(k\) {processing window}
\[
  O_k
  :=\bigl(s_{1:k},a_{1:k}\bigr),
  \qquad
  k\in\mathcal T:=\{1,\dots,T\},
\]
and feed this window to the encoder, \(q_{\phi}(z\mid O_k)\). For notational simplicity, we drop the subscript and write \(O\) once \(k\) is fixed. Before formalizing sufficiency, we impose an operational predictability requirement on the chosen window to ensure that the window $O$ is information-rich enough for inference.

\begin{assumption}[Processing window]
\label{ass:predictability}
There exists \(0<\delta<1\) such that the Bayes error of predicting
\(C\) from \(O\) satisfies
\[
  P_e(O)
  :=\!
  \min_{\hat C}\Pr\!\bigl[\hat C(O)\neq C\bigr]
  \;\le\;\delta.
\]
\end{assumption}

\noindent
Assumption~\ref{ass:predictability} does not dictate \emph{how}
the agent recovers \(C\); it merely claims that \emph{some} decoder/estimator can
achieve Bayes error at most~\(\delta\).
If the window is too short (e.g.\ \(k=1\) in systems where inertia
matters), the assumption fails, and by the
\emph{data-processing inequality} no representation learned from that
window can retrieve the missing contextual information.

\begin{remark}[Assumption~\ref{ass:predictability} is operational]
\label{rem:assumption}
Assumption~\ref{ass:predictability} converts a difficult
\textbf{dynamical} property--how strongly the context shapes transitions--into a simple \textbf{data-driven} test: the chosen window
\(O\) must carry enough statistical signal to identify \(C\) with small
error.  It imposes \emph{no} reliance on independence, stationarity, or special information structure (e.g.\ linear–Gaussian).  In practice, one can adopt an iterative ``window-growing'' strategy: start with a small
\(k\) and increase it until empirical validation (or the theoretical threshold in Lemma~\ref{lemma:fano_opt}) confirms that
\(P_e(O)\le\delta\).  This approach is operationally convenient and sidesteps model specification issues, and we revisit it in detail in later sections.
\end{remark}

\begin{definition}[Observation sufficiency]
\label{def:obs_suff}
Let the window \(O\) satisfy Assumption~\ref{ass:predictability}.
An encoder \(q_{\phi}(z\mid o)\) is {observation sufficient} for
the latent context \(C\) if
\[
    I(C;O)\;=\;I(C;Z),
\]
i.e., the code \(Z\) preserves all information about \(C\) that is
contained in the processing window \(O\).
\end{definition}

\smallskip
\begin{proposition}[Characterization via DPI]
\label{prop:obs_suff_equiv}
Given the Markov chain \(C\!\to\!O\!\to\!Z\),
\[
    I(C;O)=I(C;Z)
    \;\Longleftrightarrow\;
    I(C;O\mid Z)=0
    \;\Longleftrightarrow\;
    C\;\perp\!\!\!\perp\;O\;\mid\;Z.
\]
\end{proposition}
\begin{proof}
Apply the chain rule for mutual information and DPI. See~\citeauthor{cover1999elements} (Thm.~2.8.1) for the full argument.
\end{proof}

\smallskip\noindent
Definition~\ref{def:obs_suff} demands that \(Z\) preserve
\emph{exactly} the pieces of \(O\) that explain the latent context \(C\). Hence, any nuisance variation in \(O\) (sensor noise, irrelevant scene
details, reward noise, \textit{etc.}) is stripped away, yielding a
minimal signal for identifying the context. This sharpness is desirable both theoretically, because it
enables a clean connection to the control sufficiency criterion we
introduce next, and practically, because it concentrates the
encoder’s capacity on context relevant information, yielding a compact representation for any downstream task.

The next lemma makes this notion concrete by relating mutual information loss to Bayes error: an observation sufficient code $Z$ inherits the same Bayes error from the data window \(O\).

\begin{lemma}[Sharpness via Fano's inequality]
\label{lemma:fano_opt}
Let \(N=|\mathcal C|\) and $C$ be uniform on $\mathcal{C}$. Suppose the window \(O\) satisfies
Assumption~\ref{ass:predictability} with Bayes error
\(P_e(O)\le\delta\).

\begin{enumerate}
\item If \(q_\phi(z|o)\) is observation-sufficient, then
      \(P_e(Z)=P_e(O)\le\delta\).

\item If the encoder discards \(\varepsilon>0\) bits,
      i.e.\ \(I(C;O)-I(C;Z)=\varepsilon\),
      every decoder that sees only \(Z\) yields more error
      \[
        P_e(Z)\;\ge\;\delta+\varepsilon/\log N .
      \]
\end{enumerate}
\end{lemma}

\begin{proof}
\textbf{Fano’s inequality.}
By Assumption~\ref{ass:predictability} the window $O$ satisfies
$P_e(O)\le\delta$. Fano’s inequality for finite alphabets gives
\begin{equation}
\label{eq:fano_O}
    I(C;O)
    \;\ge\;
    (1-\delta)\log N-\log 2 .
\end{equation}

\vspace{2pt}\noindent
\textbf{Statement 1 (observation sufficiency).} Under the Markov chain $C\!\to\!O\!\to\!Z$, observation sufficiency
$\,I(C;O)=I(C;Z)\,$ implies
\(
  I(C;O\mid Z)=0
\)
(Prop.~\ref{prop:obs_suff_equiv}), or equivalently
\(
  C\perp\!\!\!\perp O \mid Z.
\)
Hence, for every $(o,z)$ with $q_\phi(z\mid o)>0$, we have
\begin{equation}
\label{eq:posterior_equal}
  p(c\mid o,z)=p(c\mid z).
\end{equation}
Using the definition of Bayes error and~\eqref{eq:posterior_equal}, we have
\[
\begin{aligned}
  P_e(O)
  &=
  \min_{\hat C}\Pr\!\bigl[\hat C(O)\neq C\bigr] \\
  &= 1-\E_O\bigl[\max_{c}p(c|O)\bigr]\\
  &= 1-\E_{O,Z}\!\bigl[\max_{c}p(c\mid O,Z)\bigr] \\
  &= 1-\E_Z\!\bigl[\max_{c}p(c\mid Z)\bigr]\\
  &= P_e(Z).
\end{aligned}
\]
Thus, if the encoder is observation sufficient, we have
$P_e(Z)=P_e(O)\le\delta$.

\vspace{2pt}\noindent
\textbf{Statement 2 (sharpness).} Suppose $I(C;O)-I(C;Z)=\varepsilon>0$.
Combining~\eqref{eq:fano_O} with this relation yields
\[
  I(C;Z)
  \;\ge\;
  (1-\delta)\log N-\log 2-\varepsilon .
\]
Applying Fano’s inequality to $Z$ gives
\[
\begin{aligned}
  P_e(Z)
  &\;\ge\;
  1-\frac{I(C;Z)+\log 2}{\log N}\\
  &\;\ge\;
  1-\frac{(1-\delta)\log N-\varepsilon}{\log N}
  =\delta+\frac{\varepsilon}{\log N},
\end{aligned}
\]
which gives statement 2.
\end{proof}

\medskip
Observation sufficiency converts a ``good-enough'' data window
\(O\) into a latent code \(Z\) that is equally informative about
the context: every bit in \(Z\) serves to identify \(C\) and none is
wasted (Lemma~\ref{lemma:fano_opt}).  For \emph{any} downstream prediction or classification task, this is ideal, because \(Z\) can replace \(C\) as a drop-in surrogate.

Policy learning, however, has a different objective.  To select
high-return actions, the agent must exploit precisely those features that drive future rewards and transitions, not merely
the bits that help ``name'' the context. The next section introduces \emph{control sufficiency}, a criterion that aligns the latent code \(Z\) with the cMDP’s value structure.

\begin{remark}[``Soft'' sufficiency via the Information Bottleneck]
\label{remark:lossy_suff}
In practice, we often prefer a small information gap
\(I(C;O)-I(C;Z)\le\epsilon\) (statement 2 in Lemma~\ref{lemma:fano_opt}), exchanging losslessness for a simpler,
potentially more robust code. Formally, this is a rate distortion problem~\citep{cover1999elements}: we seek the
\emph{cheapest} representation \(Z\) (low rate \(I(O;Z)\)) that
retains almost all context-relevant information (high \(I(Z;C)\). A convenient Lagrangian, known as the \textbf{information bottleneck
(IB)}~\citep{tishby2000information}, is
\[
  \min_{q(z|o)}\;
  \beta\,I(O;Z)\;-\;I(C;Z),
\]
where the multiplier \(\beta>0\) controls the trade-off.  A smaller \(\beta\) preserves more context information, and a larger \(\beta\) encourages additional compression.  We will leverage
this convenient IB formulation when designing the algorithmic components in later sections.
\end{remark}

\subsection{Control Sufficiency}
\label{sec:ctrl_suff}

Observation sufficiency (Section~\ref{sec:obs suff}) guarantees
that a latent code \(Z\) preserves every bit of information in the
window \(O\) that helps \emph{identify} the hidden context \(C\).  Policy learning, however, asks a different question: \emph{how well can an agent \emph{exploit} that code to choose
(near-)optimal actions?} Formally, given an encoder \(q_{\phi}\) and a \(Z\)-conditioned policy
\(\pi_{\theta}(a\mid s,z)\), we compare the achieved return to the context-aware optimum.

\begin{definition}[Control sufficiency (weak)]
\label{def:ctrl_suff}
Denote the optimal contextual return
\[
  J^\star
  :=\max_{\pi(a\mid s,c)} J(\pi),
\]
where \(J(\pi)\) is defined in~\eqref{eq:crl_objective}.
The encoder \(q_\phi\) is weakly control sufficient if there exists a
\(Z\)-conditioned policy \(\pi_\theta(a\mid s,z)\) such that
\begin{equation}
\label{eq:ctrl_suff}
   J\!\bigl(\pi_\theta\bigr)=J^\star .
\end{equation}
\end{definition}

\smallskip
By definition, control sufficiency is inherently a \emph{joint} property: the code
\(Z\) produced by the encoder is only as good as the policy that can exploit it. Any residual uncertainty about \(C\) is harmless {so long as} the accompanying policy can map
that code to actions that achieve the maximal long-term return; what matters is that the encoder–policy pair preserves the
reward-relevant information. The next proposition formalizes this intuition by showing that such a pair aligns the optimal state-action functions in an \emph{expectation-wise} sense.

\begin{proposition}[Expectation-wise value alignment]
\label{prop:value-equality}
Let $(q_{\phi},\pi_{\theta})$ be an encoder–policy pair and define
\[
      Q^{\star}(s,a,c) \quad\text{and}\quad
      Q^{\star}_{Z}(s,a,z)
\]
as the optimal $Q$-functions when the agent observes either the true
context \(c\) or only the code \(z\). Suppose
\begin{enumerate}
    \item \textbf{Weak control sufficiency}\,:  
          the $Z$-conditioned policy achieves the contextual optimum,
          \(J(\pi_{\theta}) = J^{\star}\).
    \item \textbf{Markov in augmented state}\,:  
          the augmented state \((s,z)\) is Markov, i.e.\
          \(P(s_{t+1}\mid s_{1:t},a_{1:t},c,z)
            =P(s_{t+1}\mid s_{t},a_{t},c,z)\).
\end{enumerate}
\noindent
Then, 
\begin{equation}
\label{eq:value-equality}
      Q^{\star}_{Z}(s,a,z)
      \;=\;
      \mathbb E\bigl[\,Q^{\star}(s,a,C)\mid Z=z\bigr].
\end{equation}
\end{proposition}


\begin{proof}
\textbf{Step 1: joint optimality.}
Control sufficiency implies
\(J(\pi_{\theta})=J^{\star}\).
Because the discounted return is the \emph{unique} fixed point of the
Bellman optimality operator, we have
\[
      V^{\pi_{\theta}}(s,c)=V^{\star}(s,c)
      \quad\text{and}\quad
      Q^{\pi_{\theta}}(s,a,c)=Q^{\star}(s,a,c)
      \quad\forall (s,a,c).
\]

\vspace{2pt}\noindent
\textbf{Step 2 optimality in the \((s,z)\)-augmented MDP.}  
The policy \(\pi_{\theta}\) depends only on the pair \((s,z)\), so
\[
      Q^{\pi_{\theta}}(s,a,z)
      \;=\;
      \mathbb{E}\bigl[\,Q^{\pi_{\theta}}(s,a,C)\mid Z=z\bigr]
      \;=\;
      \mathbb{E}\bigl[\,Q^{\star}(s,a,C)\mid Z=z\bigr].
\]
By the second assumption, the augmented state \(\hat s:=(s,z)\) remains Markov, so dynamic programming arguments apply.  Since \(J(\pi_{\theta})\) is already the \emph{maximal} return in this
augmented MDP, \(\pi_{\theta}\) must also be optimal for it, i.e.\
\(Q^{\pi_{\theta}}(s,a,z)=Q^{\star}_{Z}(s,a,z)\).  
Combining with the equality above yields exactly
\[
      Q^{\star}_{Z}(s,a,z)
      \;=\;
      \mathbb{E}\!\bigl[\,Q^{\star}(s,a,C)\mid Z=z\bigr],
\]
which is \eqref{eq:value-equality}.
\end{proof}

\medskip\noindent
Proposition \ref{prop:value-equality} shows that weak control sufficiency aligns state-action values only in expectation. For any code value $z$, the encoder groups a subset of contexts  
\(
    \mathcal \{c \subset \mathcal{C}: q_{\phi}(z\mid c)>0\},
\)\footnote{$q_{\phi}(z|c)$ is obtained from the window-based encoder $q_{\phi}(z|o)$ by marginalizing the contextual observation from $p(o|c)$.}
and the optimal $Q$ in the \(Z\)-augmented MDP is simply the
\emph{weighted average} of the context-optimal $Q$’s:
\[
   Q^{\star}_{Z}(s,a,z)
   \;=\;
   \sum_c
      Q^{\star}(s,a,c)\,
      \Pr\!\bigl[C=c \mid Z=z\bigr].
\]
In other words, the $Q$-function that is optimal given \(Z\) equals the \emph{average} of the context-optimal $Q$-functions under the posterior \(C\!\mid\!Z=z\). Thus, the pair $(q_{\phi}, \pi_{\theta})$ achieves the optimal \emph{expected} return even though the exact context inside each $z$-cluster can remain uncertain, i.e., the notion of ``weak'' in the definition of control sufficiency.

\begin{remark}[Weak control sufficiency in practice]\label{remark:weak in practice}
In practice, a widely used approach is to augment the policy network with an \textit{auxiliary context head}. Training typically minimizes a joint loss of the form:
\begin{equation}
\label{eq:joint_loss}
    \mathcal L_{\text{total}}
    \;=\;
    \mathcal L_{\text{RL}}
    \;+\;
    \alpha\,\mathcal L_{\text{aux}},
    \qquad \alpha>0,
\end{equation}
where \(\mathcal L_{\text{RL}}\) is the reinforcement learning objective (e.g.\ a policy-gradient or actor–critic loss), \(\mathcal L_{\text{aux}}\) is a supervised loss that tries to predict the context from shared features (e.g., typically a cross‐entropy or contrastive loss~\citep{oord2018representation}). The hyperparameter \(\alpha\) balances the two terms and \(\mathcal L_{\text{RL}}\) \textbf{dominates}. Because the loss is jointly optimized, the network is free to \emph{merge} context variations that do not change the expected
return, precisely the expectation-wise characterization by weak control sufficiency. The resulting \(Z\) is therefore sufficient for high average
performance, yet it still hides residual uncertainty about which factors in \(C\) are actually in effect.
\end{remark}


\paragraph{A stronger notion of control sufficiency.}
The weak definition characterizes the sufficiency in the sense of expected optimal return, yet it has two practical drawbacks: \textbf{(i)} high TD variance during training: a critic trained on the augmented state \((s,z)\) must learn the \emph{average} $Q$–value across all contexts aggregated by the same code value \(z\). The larger the mixture, the noisier the TD targets, \textbf{(ii)} drift under distribution shift: if the mixture \(\Pr[C\mid Z=z]\) changes (e.g.\ because the encoder is typically updated jointly with the policy), the averaged target itself drifts. We therefore introduce a stronger notion of control sufficiency.

\begin{definition}[Control sufficiency (strong)]
\label{def:ctrl_suff_strong}
The encoder \(q_\phi\) is strongly control sufficient if there exists a
\(Z\)-conditioned policy \(\pi_\theta(a\mid s,z)\) such that, for almost every \((s,a,c,z)\) with
\(q_{\phi}(z\mid c)>0\), one has
\begin{equation}\label{eq:strong_ctrl}
   Q^{\star}_{Z}(s,a,z)=Q^{\star}(s,a,c).
\end{equation}
\end{definition}

\smallskip\noindent This stronger condition rules out the averaging effects in the weak Definition~\ref{def:ctrl_suff}, requiring a point-wise match of optimal \(Q\)-values.

\begin{corollary}[Strong $\;\Rightarrow\;$ weak control sufficiency]
\label{cor:strong_implies_weak}
If the encoder \(q_{\phi}(z\mid o)\) is strongly
control sufficient in the sense of
Definition~\ref{def:ctrl_suff_strong}, then it is also weakly control sufficient.
\end{corollary}

\begin{proof}
Strong sufficiency gives the point-wise identity
\(Q^{\star}_{Z}(s,a,z)=Q^{\star}(s,a,c)\) whenever
\(q_{\phi}(z\mid c)>0\).
Hence, for every state–code pair \((s,z)\) the policy
\(\pi_\theta(a\mid s,z)=\mathbf 1\!\bigl\{a=\arg\max_{a'}Q^{\star}_{Z}(s,a',z)\bigr\}\)
coincides with an optimal context-aware action wherever the pair has non-zero probability. This policy, therefore, achieves the
optimal return \(J^{\star}\).
\end{proof}

\smallskip\noindent
Compared with the weak characterization, the point-wise criterion~\eqref{eq:strong_ctrl} is both theoretically and practically preferable, because of
\begin{itemize}
\item \textbf{Lower variance:}  
      TD targets now are \emph{point-wise}, thus critics converge faster;
\item \textbf{Robustness:}  
      Point-wise equality is robust to shifts in the distribution \(p(c)\);
\item \textbf{Modularity:}  
      The representation and control layers can be trained \emph{separately}:
      once \(q_{\phi}\) is fixed, any off-the-shelf RL algorithm can
      operate on the augmented state \((s,z)\);
\item \textbf{Clean theory:} Interpretable relationship with observation sufficiency (see subsequent sections).
\end{itemize}

\subsection{Control–Observation Sufficiency Relationship}
\label{subsec:hierarchy_summary}
We now compare the two key notions introduced so far.

\begin{itemize}
\item \textbf{Observation sufficiency}  
      (Def.~\ref{def:obs_suff}): the encoder $q_{\phi}$ keeps all relevant information about the latent context $C$ in the processing windon $O$,
      \[
          I(C;O)=I(C;Z)
          \;\Longleftrightarrow\;
          I(C;O\mid Z)=0 .
      \]
      In essence, $Z$ can ``name'' the context in any inference task.
\vspace{2pt}
\item \textbf{Control sufficiency}  
      (Def.~\ref{def:ctrl_suff}–\ref{def:ctrl_suff_strong}): there exists an
      encoder–policy pair $(q_{\phi},\pi_{\theta})$ such that
      \[
          J(\pi_{\theta})=J^{\star}\quad(\text{weak}),\qquad
          Q^{\star}_{Z}(s,a,z)=Q^{\star}(s,a,c)\;(\text{strong})
      \]
      whenever $q_{\phi}(z\mid c)>0$. Weak sufficiency aligns state-action values in expectation; strong sufficiency aligns them point-wise.
\end{itemize}

\begin{lemma}\label{lemma:hierarchy}
Every strongly control sufficient encoder is
observation sufficient.  The converse is false: one can build contextual MDPs where an encoder is observation sufficient but not strongly control sufficient.
\end{lemma}

\begin{proof}
\paragraph{(i) control $\;\Rightarrow\;$ observation.}
Assume \(q_{\phi}\) is control-sufficient but, for
contradiction, \textbf{not} observation sufficient.  
Then, \(I(C;O\mid Z)>0\), so there exist
\(c_{1}\neq c_{2}\) and \(z\) with
\begin{equation}\label{eq:thm_1}
    p(c_{1}\mid z)>0, \qquad p(c_{2}\mid z)>0.
\end{equation}
Because the two contexts differ, at least one of the following holds:

\begin{enumerate}
\item \textbf{Different rewards:}  
      There is \((\bar s,\bar a)\) s.t.\
      \(r_{c_{1}}(\bar s,\bar a)\neq r_{c_{2}}(\bar s,\bar a)\).
      Consequently,  
      \(Q^{\star}(\bar s,\bar a,c_{1})\neq
        Q^{\star}(\bar s,\bar a,c_{2})\).
\item \textbf{Different transitions:} 
      There exists \((\bar s,\bar a)\) and \(\bar s'\) with
      \(P_{c_{1}}(\bar s'\mid\bar s,\bar a)\neq
        P_{c_{2}}(\bar s'\mid\bar s,\bar a)\).
      Since the Bellman optimality operator is a contraction,
      iterating it reveals a pair \((\bar s,\bar a)\) such that  
      \(Q^{\star}(\bar s,\bar a,c_{1})\neq
        Q^{\star}(\bar s,\bar a,c_{2})\).
\end{enumerate}

\vspace{2pt}\noindent
Hence, in either case we find \((\bar s,\bar a)\) with
\begin{equation}
    Q^{\star}(\bar s,\bar a,c_{1})\neq
   Q^{\star}(\bar s,\bar a,c_{2}).
\end{equation}
However, strong control sufficiency implies the point-wise equality
\[
   Q^{\star}_{Z}(\bar s,\bar a,z)=
   Q^{\star}(\bar s,\bar a,c_{1})=
   Q^{\star}(\bar s,\bar a,c_{2}),
   \qquad\text{for all }c\text{ with }q_{\phi}(z\mid c)>0,
\]
contradicting~\eqref{eq:thm_1}.  Therefore, \(I(C;O\mid Z)=0\) and the encoder is observation sufficient.

\paragraph{(ii) Observation $\;\not\Rightarrow\;$ strong control.}
We construct a one–step contextual bandit:

\begin{itemize}
\item Contexts \(\mathcal C=\{c_{0},c_{1}\}\) with \(\Pr[c_{0}]=\Pr[c_{1}]=\tfrac12\).
\item Actions \(\mathcal A=\{a_{+},a_{-}\}\).
\item Rewards 
      \(r_{c_{1}}(a_{+})=+1,\;
        r_{c_{1}}(a_{-})=-1,\;
        r_{c_{0}}(a_{+})=-1,\;
        r_{c_{0}}(a_{-})=+1\).
\item Observations \(O\in\{o_{+},o_{-}\}\) with  
      \(\Pr[o_{+}\mid c_{1}]=p,\;
        \Pr[o_{+}\mid c_{0}]=1-p\) for any \(p\in(0,1)\setminus\{\tfrac12\}\).
\end{itemize}

\noindent
Define the deterministic encoder  
\(Z:=\Pr[C=c_{1}\mid O]\in\{p,1-p\}\).  
Because \(Z\) is a measurable function of \(O\),
\(I(C;Z)=I(C;O)\), the encoder is therefore observation-sufficient.

\noindent
The expected reward of action \(a_{+}\) conditioned on \(Z=z\) is
\[
        \mathbb{E}[r(a_{+})\mid Z=z]
        \;=\;
        z\cdot(+1)+(1-z)\cdot(-1)
        \;=\;
        2z-1 .
\]
The expected reward of \(a_{-}\) is the negative of this. Hence, the (Bayes-)optimal policy chooses \(a_{+}\) iff \(z>\tfrac12\). Because \(p\in(0,1)\setminus\{\tfrac12\}\), one of the two code values
is below \(1/2\) and the other above, so the policy is well defined. 

For either code value \(z\in\{p,1-p\}\), the policy yields
the return \(|2z-1|\).  The expected return of the best \(Z\)-conditioned policy is therefore
\[
        J^{\star}_{Z}
        \;=\;
        |2p-1|
        \;<\;
        1 
        \;=\;
        J^{\star},
\]
where the inequality is because of \(p\neq\tfrac12\).  Hence, no \(Z\)-conditioned policy can match the context-optimal return. Consequently, their point-wise optimal \(Q\)-values at \(a_{+}\) differ in sign, violating the control sufficiency. Therefore, observation sufficiency does not imply (strong) control sufficiency, completing the proof.
\end{proof}

\smallskip\noindent
Lemma \ref{lemma:hierarchy} puts the two notions of sufficiency in a strict hierarchy: strong control sufficiency implies observation sufficiency, not the opposite. From a practitioner’s viewpoint, we typically begin at the \emph{opposite} end of this chain. In most tasks, it is almost always \emph{easier} to obtain one of the weaker properties:
\begin{enumerate}
    \item \textbf{Observation sufficiency}: learn a code \(Z\) that
          maximizes the mutual information \(I(Z;C)\) or, equivalently,
          minimizes the Bayes error of predicting the context.
    \item \textbf{Weak control sufficiency}: train an encoder–policy
          pair \((q_{\phi},\pi_{\theta})\) that achieves the highest
          expected return (e.g.\ by the joint loss
          \eqref{eq:joint_loss} in Remark \ref{remark:weak in practice}).
\end{enumerate}
The key question is therefore:
\begin{quote}
\emph{How can we upgrade either of these ``easy-to-obtain'' solutions to the preferred stronger notion of control sufficiency?}
\end{quote}
\noindent
Achieving the strong property brings benefits that are highly desirable in practice, e.g., lower TD variance, robustness to
distributional shift, and modular plug-and-play training.

\smallskip
The next theorem shows that, provided the processing window is
\emph{lossless}, obtaining \emph{both} observation sufficiency and weak control sufficiency is the solution to lift the pair to the stronger notion.

\begin{theorem}
\label{thm:weak_plus_obs_equals_strong}
Let the processing window \(O\) be lossless, i.e.\
Assumption~\ref{ass:predictability} holds with \(\delta = 0\).
If an encoder–policy pair \((q_{\phi},\pi_{\theta})\) satisfies both observation sufficiency and weak control sufficiency in the sense of Definition~\ref{def:obs_suff} and~\ref{def:ctrl_suff}:
\[
    I(C;O)=I(C;Z)
    \qquad\text{and}\qquad
    J\bigl(\pi_{\theta}\bigr)=J^{\star}
\]
then the pair is \emph{strongly} control-sufficient, i.e.\
\(Q^{\star}_{Z}(s,a,z)=Q^{\star}(s,a,c)\) whenever
\(q_{\phi}(z\mid c)>0\) (Definition~\ref{def:ctrl_suff_strong}).
\end{theorem}


\begin{proof}
\paragraph{Step 1:  Deterministic context given $Z$.}
Observation sufficiency gives 
\[
    I(C;O)=I(C;Z)
    \;\Longleftrightarrow\;
    I(C;O\mid Z)=0.
\]
Using $I(C;O)=H(C)-H(C\mid O)$ and
$I(C;Z)=H(C)-H(C\mid Z)$,
\[
        H(C\mid Z)=H(C\mid O).   
\]

\noindent The window is \emph{lossless} by Assumption~\ref{ass:predictability} with
$\delta=0$ (zero Bayes
error, \(P_{e}(O)=0\)). Fano’s inequality then forces
\(H(C\mid O)=0\) (see Thm~2.10.1 in~\cite{cover1999elements}).  
Substituting this into the above equality yields \(H(C\mid Z)=0\). Zero conditional entropy means that $C$ is a deterministic function of
$Z$: there exists a measurable $f$ such that
\[
        C=f(Z)\quad\text{a.s.}
\]
In particular, for any code value $z$ with $\Pr[Z=z]>0$, the conditional
distribution \(C\mid Z=z\) is the point mass at $c=f(z)$.

\vspace{2pt}\noindent
\paragraph{Step\,2:  Weak value equality.}
Weak control sufficiency and Proposition~\ref{prop:value-equality} give,
for every $(s,a,z)$ with $\Pr[Z=z]>0$,
\[
   Q^{\star}_{Z}(s,a,z)
   \;=\;
   \E\!\bigl[Q^{\star}(s,a,C)\mid Z=z\bigr].
\]
Since \(C=f(Z)\) a.s., the conditional expectation
simplifies to
\[
   Q^{\star}_{Z}(s,a,z)
   \;=\;
   Q^{\star}\bigl(s,a,f(z)\bigr).
\]

\vspace{2pt}\noindent
\paragraph{Step 3: Point-wise matching.}
Fix $(z,c)$ with $q_{\phi}(z\mid c)>0$.  
Because $\Pr[C=c]\,q_{\phi}(z\mid c)>0$, the joint event
$\{Z=z,\,C=c\}$ has positive probability.  
Conditioning on $Z=z$ forces $C=f(z)$ almost surely, so
$c=f(z)$.  
Hence, whenever $q_{\phi}(z\mid c)>0$,
\[
      Q^{\star}_{Z}(s,a,z)=Q^{\star}(s,a,c),
\]
which is exactly the point-wise criterion in
Definition~\ref{def:ctrl_suff_strong}.  
Therefore the encoder–policy pair $(q_{\phi},\pi_{\theta})$ is strongly control sufficient, completing the proof.
\end{proof}

\smallskip
Theorem~\ref{thm:weak_plus_obs_equals_strong} makes the intuition formal: once the code \(Z\) is a perfect drop-in surrogate for the context \(C\), solving the \(Z\)-conditioned RL problem on the \((s,z)\)-augmented MDP is \textbf{equivalent} to solving the original contextual MDP. Put differently, observation sufficiency guarantees
that the agent ``\emph{knows the right question}'' (which context it is
in), and weak control sufficiency guarantees that it ``\emph{has the
right answer on average}'' (optimal expected return).  Together, these
two properties collapse the averaging and promote the representation to the \textbf{strong} notion of control sufficiency.

This result provides the theoretical backbone for the optimization
objective we develop next.  Our goal is to \emph{simultaneously} (i)
drive the encoder toward observation sufficiency and (ii) maximize the
return under the current code.

\subsection{ELBO and Information Residual}
\label{subsec:var formulation}
The results in Sections~\ref{sec:obs suff}–\ref{subsec:hierarchy_summary} tell us:
\[
    \text{Observation sufficiency}
    \;+\;
    \text{Weak control sufficiency}
    \;\Longrightarrow\;
    \text{Strong control sufficiency}.
\]
In the remainder of this section, we turn this pipeline into an objective that any standard RL algorithm can optimize. To that end, we
fix an encoder and ask the auxiliary question:
\begin{quote}
\emph{With a fixed encoder \(q_{\phi}(z\mid o)\),
how close can a \(Z\)-conditioned policy
\(\pi_{\theta}(a\mid s,z)\) get to the context–aware optimum \(J^{\star}\)?}
\end{quote}

\smallskip\noindent
Let \(p_{\mathcal D}(o)\) be the empirical distribution of
observation windows. The fixed encoder then induces a prior
\[
   q_{\phi}(z)
   \;=\;
   \int q_{\phi}(z\mid o)\,p_{\mathcal D}(o)\,do,
\]
which is the agent’s belief about \(Z\) at the start of an episode.

Given this belief, one might try to model trajectories under a \(Z\)-conditioned policy by
\[
    p_{\pi_{\theta}}(\tau\mid z)
    = p(s_{1})\,
      \prod_{t=1}^{T}
          p\bigl(s_{t+1}\mid s_{t},a_{t},z\bigr)\,
          \pi_{\theta}(a_{t}\mid s_{t},z),
\]
letting the latent code directly influence future states.  
This induces the ``optimistic dynamics \( p\bigl(s_{t+1}\mid s_{t},a_{t},z\bigr)\)'' that breaks the true causal structure (see~\citep{levine2018reinforcement}) because \(Z\) is determined \emph{after} the trajectory is generated, not before.

The simple causal correction is to let \(Z\) affect only the actions,
keeping the environment dynamics fixed:
\begin{equation}\label{eq:causal_model}
    q_{\theta}(\tau\mid c,z)
    := p(s_{1})\,
       \prod_{t=1}^{T}
           p\bigl(s_{t+1}\mid s_{t},a_{t},c\bigr)\,
           \pi_{\theta}(a_{t}\mid s_{t},z).
\end{equation}
Here, the hidden context \(c\) governs the dynamics, while the latent code \(z\) informs the agent’s actions. With this causally
correct model in place, we can work in the control-as-inference (CaI) framework.

\begin{proposition}[Contextual ELBO]\label{prop:elbo}
Let $\pi_{\theta}(a\mid s,z)$ be a $Z$-conditioned policy and
$\alpha>0$ an entropy temperature.  
With the trajectory model
$q_{\theta}(\tau\mid c,z)$ in~\eqref{eq:causal_model},
define the joint variational distribution
\[
      q(c,z,\tau)
      := p(c)\,q_{\phi}(z\mid c)\,q_{\theta}(\tau\mid c,z).
\]
Then
\begin{equation}\label{eq:contextual_elbo}
   \underbrace{\log
      \int p(c)\,p(\tau\mid c)\,e^{R(\tau)}\,dc\,d\tau}_{\text{log evidence}}
   \;\ge\;
   \underbrace{\mathbb{E}_{q}\Bigl[
        R(\tau)
        +\alpha\!\sum_{t=1}^{T}
           \mathcal H\!\bigl(\pi_{\theta}(\,\cdot\!\mid s_t,z)\bigr)
     \Bigr]}_{\displaystyle
       \mathcal J_{Z}(\theta)}
   \;-\;
   \underbrace{\bigl[I(C;\tau)-I(C;Z)\bigr]}_{\displaystyle
       \Delta I} ,
\end{equation}
where all expectations and mutual information terms are taken with respect to the joint distribution \(q(c,z,\tau)\).
\end{proposition}


\begin{proof}
Under the CaI framework, we take the unnormalized ``true'' posterior as 
\[
  p^{\star}(\tau,c)\;\propto\;
  p(c)\,p(\tau\mid c)\,e^{R(\tau)} .
\]
By Jensen’s inequality, we have
\[
  \log\!\!\int p^{\star}(\tau,c)\,d\tau\,dc
  \;\ge\;
  \mathbb E_{q}\bigl[\log p^{\star}(\tau,c)-\log q(c,z,\tau)\bigr].
\]
Expand the RHS of the inequality and insert the definitions of \(p^{\star}\) and \(q\):
\begin{align*}
 \text{RHS}
 &=
   \underbrace{\mathbb E_{q}[R(\tau)]}_{(i)}
   \;-\;
   \underbrace{\mathbb E_{q}[\log q_{\phi}(z\mid c)]}_{(ii)}
   \\[4pt]
 &\quad
   +\;
   \underbrace{\mathbb E_{q}\!\Bigl[\log p(\tau\mid c)-\log q_{\theta}(\tau\mid c,z)\Bigr]}_{(iii)} .
\end{align*}
Using the identity
\(I(C;Z)=\mathbb E_{q}[\log q_{\phi}(z\mid c)-\log q_{\phi}(z)]\), we have
\[
  -\mathbb E_{q}[\log q_{\phi}(z\mid c)]
  \;=\;
  I(C;Z)+H\!\bigl(q_{\phi}(Z)\bigr).
\]
The marginal entropy \(H(q_{\phi}(Z))\) is a constant w.r.t.~\(\theta\), so we drop it. Write  
\[\log\bigl[p(\tau\mid c)/q_{\theta}(\tau\mid c,z)\bigr]  
   =\sum_{t}\log\!\bigl[\pi(a_t\mid s_t,c)/\pi_{\theta}(a_t\mid s_t,z)\bigr].
\]
Taking the expectation and recognizing  
\(\mathcal H(\pi_{\theta}(\cdot\mid s_t,z))
  =-\mathbb E_{\pi_{\theta}}[\log\pi_{\theta}]\), we have
\[
  (iii)
  \;=\;
  \alpha
  \sum_{t=1}^{T}\mathbb E_{q}\bigl[
      \mathcal H\!\bigl(\pi_{\theta}(\cdot\mid s_t,z)\bigr)
      \bigr]
  - I(C;\tau) .
\]

\smallskip\noindent
Combining (i), (ii), (iii), we get:
\[
  \underbrace{\mathbb E_{q}[R(\tau)
  +\alpha\sum_{t}\mathbb E_{q}\!\bigl[
        \mathcal H(\pi_{\theta}(\cdot\mid s_t,z))\bigr]}_{\mathcal J_{Z}(\theta)}
  \;-\;
  \underbrace{\bigl[I(C;\tau)-I(C;Z)\bigr].}_{\Delta I}
\]
\end{proof}

\smallskip\noindent
Proposition~\ref{prop:elbo} splits the log evidence into two intuitive
terms:
\begin{itemize}
\item \textbf{$Z$-surrogate control objective
      \(\mathcal J_{Z}(\theta)\).}\;
      It is the usual entropy-regularized return, but computed under the
      \emph{frozen} code~\(Z\).  Maximizing
      \(\mathcal J_{Z}(\theta)\) with respect to the policy parameters \(\theta\)
      is the same as running a standard maximum-entropy RL
      algorithm on the \((s,z)\)-augmented MDP.  This addresses the \emph{weak} control 
      sufficiency part: ``How good can a policy be if it treats \(Z\) as the context?''.

\item \textbf{Information residual
      \(\Delta I=I(C;\tau)-I(C;Z)\).}\;
      The first term, \(I(C;\tau)\), is the total amount of context inevitably leaked by the environment through trajectories.  
      The second term, \(I(C;Z)\), is the portion of that leakage the
      encoder has already captured.  
      Their difference quantifies ``the unknown'' part that no
      \(Z\)-conditioned policy can ever exploit.
      Once the $Z$ surrogate control objective is maximized, \(\Delta I\) becomes an \emph{upper bound} on any further return improvement: shrinking it effectively raises the performance ceiling imposed by the current representation.
\end{itemize}

\begin{theorem}[Information Residual]\label{thm:residual_zero}
Fix an encoder \(q_{\phi}\) and the causal model
\eqref{eq:causal_model}.  
With \(\alpha=0\) in the ELBO,
\[
      \mathcal J_{Z}(\theta)
      :=\mathbb{E}_{q}[R(\tau)],
      \qquad
      \Delta I
      := I(C;\tau)-I(C;Z),
\]
If the Markov chain \(C\!\to\!\tau\!\to\!Z\) holds, then
\[
      \Delta I=0
      \;\Longleftrightarrow\;
      \max_{\theta}\mathcal J_{Z}(\theta)=J^{\star}.
\]
In particular, if a policy \(\pi_{\theta^{\star}}\) maximises
\(\mathcal J_{Z}\) and \(\Delta I(\phi)=0\), then
\(J(\pi_{\theta^{\star}})=J^{\star}\).
\end{theorem}

\begin{proof}
\paragraph{($\Rightarrow$)}
$\Delta I=0$, i.e.\ $I(C;\tau)=I(C;Z)$.
Under the chain $C\!\rightarrow\!\tau\!\rightarrow\!Z$ and zero residual $\Delta=0$ give $I(C;\tau\mid Z)=0$, which implies
conditional independence:
\begin{equation}\label{thm:a1}
  q(c,\tau\mid z)=q(c\mid z)\,q(\tau\mid z).
\end{equation}

\noindent Choose an optimal per-context policy
$\pi_{c}^{\star}$ with return $J_{c}^{\star}=J^{\star}$ and let
$\tilde p_{c}(\tau)$ be its trajectory distribution.  
Define the \emph{$Z$-mixture policy}
\[
   \pi_{Z}^{\star}(a\mid s,z)
   :=\sum_{c} q(c\mid z)\,\pi_{c}^{\star}(a\mid s).
\]
Under~\eqref{thm:a1}, the induced trajectory distribution is also a mixture
\[
q(\tau\mid z)=\sum_{c}q(c\mid z)\,\tilde p_{c}(\tau),
\]
then
\[
   \E_{q(\tau\mid z)}[R(\tau)]
   =\sum_{c}q(c\mid z)\,J^{\star}
   =J^{\star}\quad\forall z.
\]
Therefore, $J(\pi_{Z}^{\star})=J^{\star}$.  Any maximizer
$\theta^{\star}$ of $\mathcal J_{Z}$ achieves at least this value, but
by definition no policy can exceed $J^{\star}$, hence
$J(\pi_{\theta^{\star}})=J^{\star}$.

\paragraph{($\Leftarrow$)}
Conversely, suppose that there exists $\theta$ with $J(\pi_{\theta})=J^{\star}$.
By Proposition~\ref{prop:value-equality}, we have
\[
   Q^{\star}_{Z}(s,a,z)=\E[Q^{\star}(s,a,C)\mid Z=z].
\]
If $\Delta I>0$, some \(z\) mixes two contexts whose
\(Q^{\star}\) values differ, contradicting the equality above and the
optimal return.  Therefore \(\Delta I(\phi)=0\).
\end{proof}



\smallskip\noindent
Theorem~\ref{thm:residual_zero} recasts the ``\textit{obs.\,+\,weak
$\Rightarrow$ strong}'' argument in pure optimization form.  
Once the surrogate objective \(\mathcal J_{Z}(\theta)\) has been pushed
to its maximum, the \emph{entire} return gap is captured by the single
scalar
\(
        \Delta I.
\)
Hence, information residual \(\Delta I\) acts as a \textbf{tight certificate}: driving the
residual small is \emph{exactly} what closes the performance gap
between the \(Z\)-conditioned MDP and the true contextual optimum.

Let \(\Delta I\)  be parametrized by the encoder parameters
\(\phi\), the objective splits cleanly into two \textit{decoupled} blocks, suggesting the following alternating scheme:
\[
\begin{aligned}
\textbf{(A) Policy optimization: } &
\theta \;\leftarrow\;
      \arg\max_{\theta}\;\mathcal J_{Z}(\theta)
      &&\text{// standard MaxEnt RL on }(s,z), \\[6pt]
\textbf{(B) Residual minimization: } &
\phi \;\leftarrow\;
      \arg\min_{\phi}\;\Delta I(\phi)
      &&\text{// shrink the information gap.}
\end{aligned}
\]

\noindent
\textbf{Step~(A)} improves control performance for the \emph{current}
representation;  
\textbf{Step~(B)} tightens the representation itself,  
either by increasing \(I(C;Z)\) (better inference) or by reducing
\(I(C;\tau)\) (making behavior more context-agnostic).

\smallskip
While \(\mathcal J_{Z}\) can be maximized with any off-the-shelf RL algorithm (e.g., SAC), minimizing the information residual \(\Delta I(\phi)\) is more delicate. The next section makes explicit decomposition of $\Delta I$ and introduces practical strategies for each component.


\section{Minimizing Information Residual}
\label{sec:minimize info_residual}
Before developing concrete algorithms, it helps to decompose the information
residual
\(
   \Delta I = I(C;\tau)-I(C;Z)
\)
into two conceptually different parts.  Insert and subtract
\(I(C;O)\) to obtain
\[
   \Delta I
   \;=\;
   \underbrace{I(C;\tau)-I(C;O)}_{\textsc{Processing~Gap}}
   \;+\;
   \underbrace{I(C;O)-I(C;Z)}_{\textsc{Encoder~Gap}}.
\]
\begin{itemize}
\item \textbf{\textsc{Processing Gap}}\;
      captures the information lost when the full trajectory
      \(\tau\) is truncated to a length-$k$ window \(O\).  
      It depends on \emph{how} the data are collected and can be
      reduced empirically (Section~\ref{subsec:window_gap}).

\item \textbf{\textsc{Encoder Gap}}\;
      is the additional loss incurred when the window \(O\) is
      compressed into the latent code \(Z\) via the learned encoder
      \(q_{\phi}(z\mid o)\).  
      This is a pure representation learning problem (Section~\ref{subsec:encoder_gap}).
\end{itemize}

\subsection{Processing Gap}\label{subsec:window_gap}
Truncating the full trajectory \(\tau=(S_{1:T},A_{1:T})\) to a length-\(k\)
prefix \(O=g_k(\tau)=(S_{1:k},A_{1:k})\) inevitably discards context
information carried by the \emph{tail} indices
\(\mathcal T_{\!{\setminus k}}\!:=\{k+1,\dots,T\}\).
By the chain rule of mutual information~\cite{cover1999elements} and the Markov property of the
MDP, one can show that this gap is the information ``tail'':
\begin{equation}\label{eq:window-gap}
   \textsc{processing gap}\;:=\;I(C;\tau)-I(C;O)
   \;=\;
   \sum_{t\in\mathcal T_{{\setminus k}}}
   I\bigl(C;\,(S_{t},A_{t}) \,\big|\, S_{1:t-1},A_{1:t-1}\bigr).
\end{equation}

\noindent
Every omitted index contributes the conditional term
\(I\bigl(C;(S_t,A_t)\mid S_{1:t-1},A_{1:t-1}\bigr)\)
to the processing gap.  
Whether the tail in~\eqref{eq:window-gap} dies out or not is an intrinsic property of the controlled process:
\begin{itemize}
\item \textbf{Long-memory (slow causal propagation):}  
      If the context must travel a long causal chain
      before showing up in \((S_t,A_t)\), every extra step still
      reveals new information.  The gap shrinks only when the \emph{entire} trajectory is kept in the window.

\item \textbf{Fast-mixing:}  
       In many controlled systems, the conditional term in
      \eqref{eq:window-gap} decays exponentially after a mixing time
      \(\tau_{\mix}\); then a modest window
      \(k\!\approx\!T-\tau_{\mix}\) already makes the tail negligible~\citep{bradley2005basic}.

\item \textbf{Stationary \& ergodic:}  
      When the triple \((C,S_t,A_t)\) is jointly stationary and ergodic, the conditional MI settles to a \emph{constant rate} \(I_{\infty}(C;S,A)\). Every step beyond the chosen window leaks the same small amount of context information, so the entire tail contributes a flat penalty \((T-k)I_{\infty}(C;S,A)\) (by Shannon’s asymptotic equipartition property (AEP) ~\citep{cover1999elements}).
\end{itemize}

\medskip
\noindent
Bounding the tail term in \eqref{eq:window-gap} is usually infeasible: the leakage rate depends on mixing or ergodic
constants of the \emph{controlled} stochastic processes, a property that is both
task-specific and rarely available in closed form.  Rather than pursue such model-centric analysis, we re-invoke our operational Assumption~\ref{ass:predictability}.  That single data–driven condition turns into an immediate information bound.

\begin{corollary}\label{coro:fano}
If Assumption~\ref{ass:predictability} holds and
\(|\mathcal C|=N\), then
\[
   \textsc{Processing Gap}\;\le\;
   h_2(\delta)\;+\;\delta\log(N-1),
\]
where $h_2(\delta)=-\delta\log\delta-(1-\delta)\log(1-\delta)$ is the binary entropy.
\end{corollary}

\begin{proof}
For a finite context alphabet, Fano’s inequality gives
\[
H(C\!\mid O)\le h_2(\delta)+\delta\log(N-1).
\]
Given the definition $I(C;O) = H(C) - H(C|O)$, we have
\[
I(C;O) \ge H(C) - h_2(\delta) - \delta\log(N-1).
\]
Combining it with the fact $I(C;\tau) \leq H(C)$ yields the stated bound.
\end{proof}

\smallskip\noindent
A small Bayes error on the window translates directly
into a small processing gap: when $P_e(O)=\delta$, at most
$\mathcal O(\delta)$ bits of context remain hidden in the tail. Thus, by task-specific window-engineering, one obtains an information-theoretic guarantee.

\subsection{Encoder Gap}\label{subsec:encoder_gap}
Once the \textsc{Processing Gap} has been engineered away, the remaining term in $\Delta I$ is the
\[
  \textsc{Encoder Gap}
  \;:=\;
  I(C;O)-I(C;Z),
\]
\emph{i.e.}, the context information present in the window \(O\) but
discarded by the encoder \(q_{\phi}(z\mid o)\).
Because \(C\!\rightarrow\!O\!\rightarrow\!Z\) is a Markov chain, the
gap is \textbf{exactly} the violation of \emph{observation sufficiency}
(Definition~\ref{def:obs_suff}).

\paragraph{Information Bottleneck View.} The term $I(C;O)$ is fixed once the window is chosen, so shrinking the
encoder gap means increasing $I(C;Z)$ while paying a rate
penalty~$I(Z;O)$ (see Remark~\ref{remark:lossy_suff}).  This penalty view leads to the IB objective\footnote{We use the
compression–weighted convention common in modern IB literature.}
\begin{equation}\label{eq:ib}
   \mathcal L_{\IB}(\phi)
   := \beta\,I(Z;O)\;-\;I(C;Z),
   \qquad 0<\beta<1 .
\end{equation}
Here, \(\beta\) controls how costly it is to keep extra bits from the window once sufficiency is reached.

\begin{proposition}[Observation sufficiency and the IB loss]
\label{prop:ib_eps_swap}
Let $C\!\to\!O\!\to\!Z$ and define
\(
  \varepsilon(\phi):=I(Z;O)-I(C;Z)\ge 0.
\)
Then for any encoder $q_{\phi}$ and $0<\beta<1$, we have
\begin{equation}\label{eq:ib_bound}
   0\;\le\;\varepsilon(\phi)
   \;\le\;
   \beta^{-1}\!\bigl[
       \mathcal L_{\IB}(\phi)+(\beta-1)\,I(C;O)
   \bigr].
\end{equation}
Moreover,
\[
   \varepsilon(\phi)=0
   \;\Longleftrightarrow\;
   \mathcal L_{\IB}(\phi)=-(1-\beta)\,I(C;O),
\]
so the encoder is observation-sufficient
iff it attains the global minimum of the IB loss.
\end{proposition}

\begin{proof}
Write $I(Z;O)=I(C;Z)+\varepsilon$.  Then
\[
  \mathcal L_{\IB}
  =\beta\bigl[I(C;Z)+\varepsilon\bigr]-I(C;Z)
  =(\beta-1)I(C;Z)+\beta\varepsilon.
\]
Rearranging
\[
  \varepsilon
  =\beta^{-1}\bigl[
      \mathcal L_{\IB}-(\beta-1)I(C;Z)
   \bigr].
\]
Since $I(C;Z)\le I(C;O)$ (by DPI) and $\beta-1<0$, replacing $I(C;Z)$ by
$I(C;O)$ yields the upper bound in~\eqref{eq:ib_bound}.  
If $\varepsilon=0$, then $I(C;Z)=I(C;O)$, so
$\mathcal L_{\IB}=-(1-\beta)I(C;O)$.  
Conversely, the IB loss achieves this value \emph{only} when
$\varepsilon=0$, completing the equivalence.
\end{proof}

\smallskip\noindent
Proposition~\ref{prop:ib_eps_swap} rewrites the encoder gap
\(I(C;O)-I(C;Z)\) as a single IB loss
\(\mathcal L_{\IB}(\phi)\). Its global minimum is \emph{equivalent} to observation sufficiency, and
the bound~\eqref{eq:ib_bound} tells us exactly how far any iterate remains from that target. Because \(\mathcal L_{\IB}\) depends only on the encoder and the
windowed data, it provides a \emph{stable, low-variance} learning signal
that keeps \(Z\) from collapsing (even while policy gradients or critic estimates fluctuate during training.)

\paragraph{IB annealing.} Unlike classical IB representation learning, $\beta$ here governs not only the compression versus prediction trade-off but also how \emph{robust} the learned code $Z$ will be for the policy (See Table~\ref{tab:beta_roles}). A practical annealing schedule exploits the effects of $\beta$ in two regimes: start with a small~\(\beta\) to encourage rich features for exploration, then gradually raise \(\beta\) so that the code $Z$ converges toward a \emph{minimal} observation sufficient representation for robustness.
\begin{table}[ht]
\centering
\footnotesize
\caption{Effect of the parameter $\beta$ in IB objective
$\mathcal L_{\IB}(\phi)$ ($0<\beta<1$).}
\vspace{2pt}
\renewcommand{\arraystretch}{1.1}
\begin{tabular}{@{}p{1cm}|p{4cm}|p{5cm}|p{4cm}@{}}
\toprule
\makecell{\textbf{$\beta$}} 
& \makecell{\textbf{IB update}\\\textbf{focus}} 
& \makecell{\textbf{Code ($Z$)}\\\textbf{properties}} 
& \makecell{\textbf{Control}\\\textbf{implication}} \\
\midrule
\centering
\(\beta\!\downarrow\!0^{+}\) 
& \makecell{Maximizing \(I(C;Z)\) \\\ (compression nearly free) }
& \makecell{Observation‐sufficient \\\ \emph{plus} extra ``beyond-$C$'' bits} 
& Rich features but risk overfitting \\[4pt]
\midrule
\centering
\(\beta\!\uparrow\!1^{-}\) 
& \makecell{Minimizing \(I(Z;O)\) \\\ (Strong compression)} 
& \makecell{\emph{Minimal} observation sufficient} 
& Minimal abstraction improves robustness but may remove helpful heuristics \\
\bottomrule
\end{tabular}
\label{tab:beta_roles}
\end{table}

\paragraph{Putting the pieces together.} With the processing gap bounded and the information bottleneck view on the encoder gap, the ELBO admits an envelope form
\begin{equation}\label{eq:obj}
    \mathrm{ELBO}
    \;\;\ge\;\;
    \underbrace{\mathcal J_{Z}(\theta)}_{\text{policy term}}
    \;-\;
    \underbrace{\min_{\phi}\mathcal L_{\IB}(\phi)}_{\text{encoder term}},
\end{equation}
where \(\mathcal L_{\IB}(\phi)=I(C;O)-I(C;Z)\) is the IB loss and
the processing constant has been absorbed into the bound.
Equation~\eqref{eq:obj} now cleanly \emph{separates} the optimization over policy parameters \(\theta\) from that over encoder parameters \(\phi\).

\smallskip
Crucially, the term \(\displaystyle\min_{\phi}\mathcal L_{\IB}\)
is \emph{NOT} a soft regularizer; it forces observation
sufficiency before the policy update. Intuitively, the nested optimization~\ref{eq:obj} searches for a
decoder–policy pair whose latent \(Z\) renders the contextual RL problem
\emph{separable}.  Once the IB loss hits its minimum
(\(\varepsilon(\phi)=0\)): the encoder renders the MDP on \((s,z)\) observed, so maximizing \(\mathcal J_{Z}\) solves the original contextual RL problem (approximately).

\medskip
\noindent\textbf{Preview: replay gap.}
The \textsc{Processing} and \textsc{Encoder} gaps are \textbf{\emph{intrinsic}} components of the information residual: they would exist even if the policy and encoder were updated on fresh data from the current policy distribution. In practice, however, both RL and representation learning implementations rely on a replay buffer that stores trajectories generated by \emph{past} policies.
The resulting mismatch introduces an additional \textbf{\emph{extrinsic}} error that we term as the Replay Gap.

\subsection{Replay Gap: information leakage due to experience replay}
\label{subsec:hidden_replay_gap}
Experience replay is the ``working heart'' of large-scale RL, but its statistical consequences differ sharply for \emph{value} estimates and \emph{information-theoretic} objectives.

\paragraph{Value estimate.} The reuse is safe for the \emph{value} terms because the
returns $\E_{q_\theta}[R_c]$, state values $V_\theta(s,c)$ and
advantages $A_\theta(s,a,c)$ are \textbf{linear} in the trajectory law $q_\theta(c,\tau)$.  Re-using off-policy data
\((c,\tau)\sim q_{\mathcal B}\) is therefore unbiased under standard importance weights \(w=q_{\theta}/q_{\mathcal B}\). Techniques such as clipping or V-trace style corrections keep the variance finite~\citep{precup2000eligibility,munos2016safe,espeholt2018impala}.

\paragraph{MI estimate.}
The mutual information is a
\emph{non–linear} functional of the underlying trajectory law
\(q\). Evaluating it with stale data incurs the \textbf{replay gap}:
\begin{equation}\label{eq:replay gap}
  \Gamma_{\textsc{replay}}
  \;:=\;
  \bigl|\,I_{q_\theta}(C;X)-I_{q_{\mathcal B}}(C;X)\bigr|,
\end{equation}
where \(X\) can be the window \(O\), the code \(Z\), or any further feature derived from the trajectory. This implies that even a small drift in likelihood ratios (i.e., \(\log w\)) perturbs every log term in a multiplicative, non-canceling way, so \emph{stale} context clues can ``leak'' into the estimator and be spuriously exploited by
the policy.

A cheerful fact about the replay gap is that when the behavior policy
drifts only slightly from one update to the next,
$q_{\!\mathcal B}\!\approx\!q_{\theta}$ and, under \emph{any} sensible
statistical metric, the bias
$\Gamma_{\textsc{replay}}$ must be small. We now make this intuition concrete:

\begin{proposition}[Replay gap bound under importance weights]
\label{prop:local_replay_gap}
Let $q_{\theta}(c,\tau)$ be the current on-policy distribution
over contexts $c\in\mathcal C$ ($|\mathcal C|=N > 2$) and trajectories
$\tau$, and let $q_{\mathcal B}(c,\tau)$ denote the replay-buffer
distribution.  
Assume the trajectory–level importance weights are clipped
\[
   w(c,\tau)
   :=\frac{q_{\theta}(c,\tau)}{q_{\!\mathcal B}(c,\tau)}
   \;\in\;[1-\epsilon,\,1+\epsilon],
   \qquad 0<\epsilon<1 .
\]

\noindent
An auxiliary (processed) variable $X$ is generated from $\tau$ by a Markov kernel. Then the replay gap satisfies
\[
   \Gamma_{\textsc{replay}}
   =|I_{q_\theta}(C;X)-I_{q_{\!\mathcal B}}(C;X)|
   \;\le\;
   \frac{\epsilon}{2}\log(N-1)+h_2(\frac{\epsilon}{2}),
\]
where $h_2$ is the binary entropy.
\end{proposition}

\begin{proof}
\textbf{Step 1. From importance weights to a TV bound.}
Recall the total variation (TV) distance,
\(
   \|P-Q\|_{\mathrm{TV}}
   =\frac12\sum_{c,\tau}|P(c,\tau)-Q(c,\tau)|.
\)
With the clipped likelihood ratio
\(
   w(c,\tau)=q_{\theta}(c,\tau)/q_{\mathcal B}(c,\tau)
   \in[1-\epsilon,1+\epsilon],
\)
we have
\[
   |q_{\theta}(c,\tau)-q_{\!\mathcal B}(c,\tau)|
   =|w(c,\tau)-1|\,q_{\mathcal B}(c,\tau)
   \;\le\;\epsilon\,q_{\mathcal B}(c,\tau).
\]
Summing over $(c,\tau)$ gives a TV bound:
\[
   \|q_{\theta}-q_{\!\mathcal B}\|_{\mathrm{TV}}
   \le\frac12\sum_{c,\tau}\epsilon\,q_{\mathcal B}(c,\tau)
   =\frac{\epsilon}{2}.
\]

\smallskip
\textbf{Step 2. Propagate the TV bound via data processing.}  
Introduce the \emph{processing channel}
\[
   K:\;(c,\tau)\;\longmapsto\;(c,x),
   \qquad
   K(c,x\mid c,\tau)=\kappa(x\mid\tau)\,\mathbf 1_{\{c\}},
\]
where the context label is copied, while \(x\) is obtained through
the Markov kernel \(\kappa(\cdot\mid\tau)\). Because the \emph{same} channel is applied to both measures, the strong
data–processing inequality for TV
\citep{polyanskiy2017strong} gives
\[
  \|q^{CX}_{\theta}-q^{CX}_{\mathcal B}\|_{\mathrm{TV}}
  \le\eta_K \|q_{\theta}-q_{\mathcal B}\|_{\mathrm{TV}}
  \le \|q_{\theta}-q_{\mathcal B}\|_{\mathrm{TV}},
\]
where the $\eta_K \leq 1$ is the Dobrushin coefficient.

\smallskip
\textbf{Step 3.  Continuity of entropy.}  
Write
$I(C;X)=H(C)-H(C\mid X)$.  The marginal entropy $H(C)$ is fixed by definition, so only the conditional term varies.
For two joint laws $P,Q$ on a product alphabet
$\mathcal C\times\mathcal X$ with
$\mathrm{TV}(P,Q)\le\Delta \leq 1 - \frac{1}{|\mathcal{C}|}$, the continuity of (conditional) entropy gives~\cite{cover1999elements, berta2025continuity}
\[
  |H_P(C\mid X)-H_Q(C\mid X)|
  \;\le\;
  \Delta\,\log(|\mathcal C|-1)+h_2(\Delta).
\]
Taking $\Delta=\epsilon/2$ yields the inequality:
\[
   \Gamma_{\textsc{replay}}
   =|I_{q_\theta}(C;X)-I_{q_{\!\mathcal B}}(C;X)|
   \;\le\;
   \frac{\epsilon}{2}\log(N-1)+h_2(\frac{\epsilon}{2}).
\]
\end{proof}

\noindent
Proposition~\ref{prop:local_replay_gap} says that the replay bias
\(\Gamma_{\textsc{replay}}\) grows \emph{super-linearly} with the policy
drift~\(\epsilon\) and is amplified by a
\(\log \bigl(|\mathcal C|\bigr)\) factor.  When the context alphabet is
large (or effectively continuous), this amplification is not negligible.  Two practical rules follow immediately:
\begin{enumerate}
\item \textbf{Keep the observation module \emph{``on-policy''}:}
Train the encoder $q_\phi(z|o)$ and sample the latent code $Z$ from the induced prior only on freshly collected trajectories (i.e., small distributional mismatch). This ensures the optimizable encoder gap, not replay bias, dominates the information residual.
\item \textbf{Control the \emph{effective} context alphabet:}
Limit $N$ during optimization, which can be done, e.g., via coarse quantization, curriculum strategy that progressively refines context resolution, or task-specific grouping of similar contexts. A smaller effective alphabet provably reduces the $\log N$ amplification in the replay bound.
\end{enumerate}

\noindent
Combined with the processing and encoder bounds, the replay result completes the error decomposition for the information residual $\Delta I$: every term now is either
\textit{intrinsically limited} (pre-processing),
\textit{directly optimized} (IB loss),
or \textit{provably bounded} (controlled replay bias).



\section{Algorithm: Bottlenecked Contextual Policy Optimization}
\label{sec:bcpo}
\subsection{From theory to algorithm}\label{subsec:roadmap}
Before detailing the implementation, we summarize the theoretical
results of Sections~\ref{subsec:window_gap}–\ref{subsec:hidden_replay_gap}
into a \textbf{roadmap} that clarifies

\begin{enumerate}
\item \emph{what} must be optimized,
\item \emph{why} those terms suffice for near–optimal control, and
\item \emph{how} the three information gaps are detected and mitigated.
\end{enumerate}

\paragraph{Error decomposition.}
Proposition~\ref{prop:ib_eps_swap}, Corollary~\ref{coro:fano}, and
Proposition~\ref{prop:local_replay_gap} tighten the ELBO into
\[
  \Delta I
  \;=\;
  \boxed{\textsc{Processing}}
  \;+\;
  \boxed{\textsc{Encoder}}
  \;+\;
  \boxed{\textsc{Replay}} ,
\]
clearly isolating \emph{where} optimization effort is required.

\paragraph{Nested objective \& “inner–first’’ rule.}
Substituting this decomposition yields the two–level programming
\[
  \max_{\theta}\Bigl[\,
     \mathcal J_{Z}(\theta)
     \;-\;
     \underbrace{\min_{\phi}\,
                  \mathcal L_{\mathrm{IB}}(\phi)}_{\text{inner}}
  \Bigr] \;+\;
     \underbrace{\Delta_{\textsc{processing}}
     \;-\;
     \Gamma_{\textsc{replay}}}_{\text{bounded}},
\]
where both residual terms are \emph{a-priori} bounded during
optimization.  Minimizing \(\mathcal L_{\IB}\) first drives the code \(Z\) to {observation sufficiency}; the outer maximization then makes the pair ($Z$-conditioned policy, encoder) to be {control sufficient} up to the bounded errors.

\paragraph{Gap diagnostics.}
Each term is paired with an explicit monitoring:

\smallskip
\noindent
\begin{tabular}{@{}p{0.28\textwidth}@{\hspace{0.3em}}p{0.68\textwidth}@{}}
\textsc{Processing gap} &
The IB objective already tracks the empirical mutual
information $\widehat I_\phi(Z;C)$.
If the running average $\bar I_\phi(Z;C)$ ever drops
\emph{below} the Bayes–optimal lower bound from
Lemma~\ref{lemma:fano_opt}, the observation window is too
narrow (not sufficient).\\[2pt]
\textsc{Encoder gap} &
Track $\mathcal L_{\mathrm{IB}}(\phi)$.  Convergence to the
$\beta$–dependent minimum signals that~$Z$ is
observation–sufficient at the current compression level.\\[2pt]
\textsc{Replay gap} &
Clip importance weights to
$w\!\in\![1-\epsilon,\,1+\epsilon]$ and restrict
gradients to the newest data.
A coarse and exploitive curriculum over the context set
keeps the effective alphabet size small.
\end{tabular}

\paragraph{Annealing the bottleneck.}
Table~\ref{tab:beta_roles} shows that a small compression weight
$\beta\!\ll\!1$ permits \emph{redundant} codes that accelerate
exploration while the buffer is nearly on–policy. We therefore anneal
$\beta_t\uparrow1$ in order for a smooth transition from
\emph{redundant–} to \emph{minimal–code} training.

\subsection{BCPO: Bottlenecked Contextual Policy Optimization}
\label{subsec:bcpo}
Algorithm~\ref{alg:bcpo} instantiates the roadmap in a
\emph{single training loop} that couples a variational information
bottleneck (VIB)~\citep{poole2019variational, alemi2016deep} with Soft Actor–Critic (SAC)~\citep{pmlr-v80-haarnoja18b}.

\begin{algorithm}[t]
\caption{\textbf{BCPO}.  \emph{Notation:}
$o$: size–$k$ observation window;
$c$: task/context drawn from curriculum sampler $p_{\mathrm{ctx}}$;
$z$: latent code, $z\!\sim\!q_\phi(z\!\mid\!o)$;
$\mathcal D_{\mathrm{rec}}$: freshest $\gamma|\mathcal D|$ tuples;
$\beta_t$: IB weight, annealed by $\Delta_\beta$;
$N_{\mathrm{enc}}\!\gg\!N_{\mathrm{rl}}$: enforces the inner–first rule.}
\label{alg:bcpo}
\begin{algorithmic}[1]\small
\Require \textbf{Environment / Sampler:}
        context set $\mathcal C$; sampler $p_{\mathrm{ctx}}$
\Require \textbf{Parametric models:}
        encoder $q_\phi(z\!\mid\!o)$;
        MI estimator $I_\phi(Z;C)$;
        actor $\pi_\theta$;
        critics $Q_{\psi_{1,2}}$
\Require \textbf{Data buffers:}
        replay $\mathcal D$;\, recency ratio $\gamma$
\Require \textbf{Hyper–parameters:}
        warm-up episodes $W$;\, epochs $T$;\,
        encoder steps $N_{\mathrm{enc}}\!\gg\!N_{\mathrm{rl}}$;\,
        learning rate $\eta_\phi$;\,
        IB schedule $\{\beta_t\}_{t=0}^{T}$
\Statex\textcolor{purple}{\bf Warm–up (fill buffer \& pre–train encoder)}
\For{$e=1$ \textbf{to} $W$} \Comment{random policy}
    \State collect episode $\{(s,a,r,s',o',c)\}$ and push to $\mathcal D$
\EndFor
\State optimize $\phi$ on $\mathcal D$ for $N_{\mathrm{enc}}$ steps
\Statex\textcolor{purple}{\bf Main loop}
\For{$t=1$ \textbf{to} $T$}
    \State $c\sim p_{\mathrm{ctx}}$;\, reset $M_c$
    \Repeat
        \State encode $z\!\sim\!q_\phi(\,\cdot\mid o)$;\,
               take $a\!\sim\!\pi_\theta(\,\cdot\mid s,z)$
        \State env.\ step $\rightarrow(s',r,o')$;\,
               store $(s,a,r,s',o',c)$ in $\mathcal D$ \Comment{$z$ \underline{\emph{not}} stored}
    \Until{episode ends}
    \State $\mathcal D_{\mathrm{rec}}\leftarrow$ newest $\gamma|\mathcal D|$ tuples
    \Statex\textcolor{orange}{\bf\underline{Inner step: Information bottleneck loss}} \Comment{Observation suff.}
    \For{$i=1$ \textbf{to} $N_{\mathrm{enc}}$}
        \State sample $B\subset\mathcal D_{\mathrm{rec}}$
        \State $\displaystyle
               \phi \leftarrow \phi -
               \eta_\phi\nabla_\phi\!\bigl[
               \mathcal{L}_{\mathrm{IB}}(\phi)
               \bigr]_{B}$
    \EndFor
    \Statex\textcolor{cyan}{\bf\underline{Outer step: MaxEnt RL with updated codes}}
    \For{$j=1$ \textbf{to} $N_{\mathrm{rl}}$}
        \State sample $B\subset\mathcal D$;\, re-encode each $(o,c)\in B$
        \State update $Q_{\psi_{1,2}}$ and $\pi_\theta$ with SAC losses
    \EndFor
    \State $\beta_{t+1}\gets\beta_t+\Delta_\beta$ 
\EndFor
\end{algorithmic}
\end{algorithm}

\paragraph{Context curriculum \& online encoding.}
Beyond the core steps supported by the roadmap in Section~\ref{subsec:roadmap}, BCPO adds two pragmatic components that
greatly improve sample efficiency:

\begin{enumerate}
\item \textbf{Context sampler.}  
      When $\mathcal C$ is large or continuous, we \emph{coarse–grain}
      it.  During training, all observed contexts are clustered into
      $N_{\mathrm{bin}}(t)$ cells with a user‐defined resolution~$\varepsilon_{\mathrm{ctx}}$.  
      The sampler $p_{\mathrm{ctx}}$ samples a cell in
      proportion to its running mean return  
      (i.e., a simple \emph{easy-first} curriculum), then draws a context from that cell.  
      This keeps the \emph{effective} alphabet
      $N_{\mathrm{eff}}\!=\!N_{\mathrm{bin}}(t)$ small while gradually exposing harder contexts as performance improves.

\item \textbf{On-the-fly encoding.}  
      Latent codes are \emph{never} stored in $\mathcal D$.  
      Whenever a tuple $(o,c)$ is sampled, we re-encode
      $z\!\sim\!q_\phi(\,\cdot\mid o)$ with the \emph{current} encoder.
      This exploits the separation structure: the actor–critic enjoys
      standard off-policy updates, yet always sees fresh, nearly
      sufficient state information.
\end{enumerate}

\paragraph{Inner step: optimizing the information bottleneck.}
Each of the $N_{\mathrm{enc}}$ encoder minimizes the variational IB loss 
\begin{equation}
     \mathcal L_{\IB}(\phi)
     = \beta_t\,D_{\mathrm{KL}}\!\bigl[
         q_\phi(z\mid o)\,\big\|\,r(z)
       \bigr]\;
       -\;
       \hat I_{\mathrm{NCE}}(C;Z)
   \qquad \beta_t\!\in\!(0,1),
\end{equation}
where  
\begin{itemize}
\item $r(z)=\mathcal N(0,I)$ is an isotropic Gaussian prior;
\item $D_{\mathrm{KL}}$ is evaluated in closed form under the
      re-parameterization 
      $z=\mu_\phi(o)+\sigma_\phi(o)\odot\epsilon$,
      $\epsilon\!\sim\!\mathcal N(0,I)$~\cite{kingma2013auto};
\item $\hat I_{\mathrm{NCE}}$ is an InfoNCE estimator~\citep{oord2018representation}.  
      Because the true context $c$ is hidden, we implement the contrastive
      $(\text{positive},\text{negative})$ pairing in a
      \textit{semi-supervised} fashion:  
      two codes $z^{+},z^{-}$ are drawn from different mini-batch
      windows $o^{+},o^{-}$ that come from \emph{different} contexts $c^{+}, c^{-}$. The loss, therefore, maximizes the agreement of codes within the same
      context while minimizing agreement across contexts.
\end{itemize}
All estimator details are in the Appendix.

\paragraph{Outer step: optimizing the augmented SAC.}
SAC proceeds on the augmented state $\tilde s=(s,z)$.  Twin critics
$Q_{\psi_{1}},Q_{\psi_{2}}$ and a Gaussian actor $\pi_\theta$ are
trained using
\begin{align}
\mathcal L_{\text{critic}}
  &=
  \E_{\tilde s,a,r,\tilde s'}
    \Bigl[\!
      \bigl(Q_{\psi}(\tilde s,a)
            -[\,r+\gamma\,\hat V_{\bar\psi}(\tilde s')]\bigr)^{\!2}
    \Bigr], \\[2pt]
\mathcal L_{\text{actor}}
  &=
  \E_{\tilde s}
    \Bigl[
      \alpha\,\mathcal H\!\bigl(\pi_\theta(\cdot\mid\tilde s)\bigr)
      -Q_{\psi}(\tilde s,a)\Bigm|\,
      a\!\sim\!\pi_\theta(\cdot\mid\tilde s)
    \Bigr],
\end{align}
where $\hat V_{\bar\psi}$ is the averaged target
network that stabilizes value targets
\citep{munos2016safe, lillicrap2015continuous}.  Gradients are \emph{stopped} through $z$ in $\mathcal L_{\text{actor}}$ to avoid
feedback loops with the concurrently updating encoder. Implementation details, including hyperparameter values and baseline performance, are in the Appendix.

\section{Experiments}\label{sec:expr}
We empirically assess \textbf{BCPO} on a suite of continuous–control
benchmarks designed with latent physical parameters.  The evaluation is
conducted to answer four guiding questions:
\begin{itemize}
\item \textbf{Sample efficiency \& Performance.}  
      How does BCPO’s learning curve compare with other RL baselines?
\item \textbf{Out-of-training generalization.}  
      Does the dual \emph{observation–policy} structure transfer to
      \emph{unseen} context regimes without fine-tuning?
\item \textbf{IB-annealing.}  
      How sensitive is performance to the heuristic $\beta$-schedule
    that gradually tightens the information bottleneck?  What role
    does annealing play in policy optimization?
\item \textbf{Theory in practice.} To what extent do the empirical results and ablations validate the theoretical results, and what new insights do they offer?
\end{itemize}

\subsection{Baselines and experimental set-up} \label{subsec:benchmarks}

\textbf{Baselines.}  We benchmark against \textbf{seven} carefully
chosen algorithms, each representative of a strategy for
handling context variation.  All baselines inherit the \emph{same}
SAC backbone, optimizer, and hyperparameters in order to isolate the effect of their context–handling logic.  For clarity, we partition them into two families:
\begin{itemize}
\item \textbf{Implicit-context policies.}  
      These methods \emph{embed} context variability directly in the
      policy by broad exposure, but
      \emph{do not learn} an explicit observation module.
      \begin{itemize}
      \item \textbf{Domain Randomization (DR)}
            \citep{tobin2017domain,Peng2018}: uniform sampling from the
            full continuous context range every episode.
      \item \textbf{Round-Robin (RR)}
            \citep{speck2021learning}: cyclic traversal of a fixed,
            discrete context set.
      \item \textbf{Self-Paced Context Expansion (SPaCE)}
            \citep{eimer2021self}: value-guided curriculum over a
            discrete context grid.
      \item \textbf{Self-Paced DRL (SPDRL)}
            \citep{klink2020self}: curriculum learning on a continuous
            context distribution.
      \end{itemize}

\item \textbf{Explicit-context policies.}  
      These methods condition the policy with an \emph{external} context
      signal:
      \begin{itemize}
      \item \textbf{Parameter Identification (MSE)}
            \citep{kumar2021rma}: jointly trains a system-ID network
            (i.e., minimizing MSE loss) and an estimated context-conditioned policy.
      \item \textbf{Observation Augmentation (ObsAug)}
            \citep{yang2021learning,cao2022cloud}: concatenates a
            fixed \(k\)-step \((s,a,r)\) history to the current state. In the language of our framework, this is a \emph{pure
            window} method, which effectively skips the encoding stage.
    
      \item \textbf{Context-based Meta RL (PEARL)} \citep{rakelly2019efficient}: a
            context-based meta-RL algorithm that infers a latent
            variable from recent transitions for rapid policy
            adaptation. Architecturally, it's similar to BCPO but
            \emph{optimizes} the context posterior with a KL
            \emph{regularizer} rather than an information-residual
            objective. In the terminology of our framework, PEARL collapses the \textsc{Processing Gap} into an \emph{overly simplified} form: it treats the $k$ tuples inside the observation window \(O\) as a \emph{mixture-of-experts} factorization. It, therefore, discards temporal correlations that carry context information, so the true processing gap is underestimated and the residual bias remains uncorrected.
      \end{itemize}
\end{itemize}

\vspace{3pt}\noindent
\textbf{Implementation details.} All explicit-context baselines \textbf{MSE,
ObsAug, PEARL} and our \textbf{BCPO} share the same
\(\,k=10\) step context window containing state, action, and reward
information.  \textbf{PEARL} and \textbf{BCPO} also employ the \emph{same} probabilistic
encoder architecture and isotropic Gaussian prior to ensure a fair comparison. All implementation details are provided in the Appendix. 

\begin{figure}[b]
    \centering
    \includegraphics[width=0.19\linewidth]{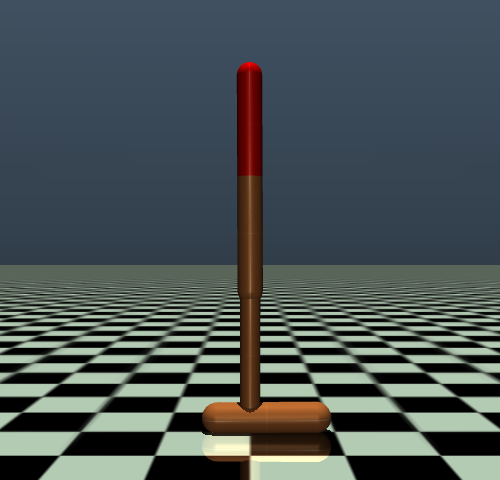}
    \includegraphics[width=0.19\linewidth]{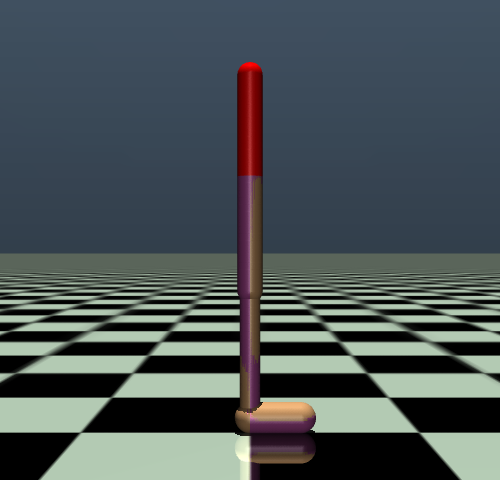}
    \includegraphics[width=0.19\linewidth]{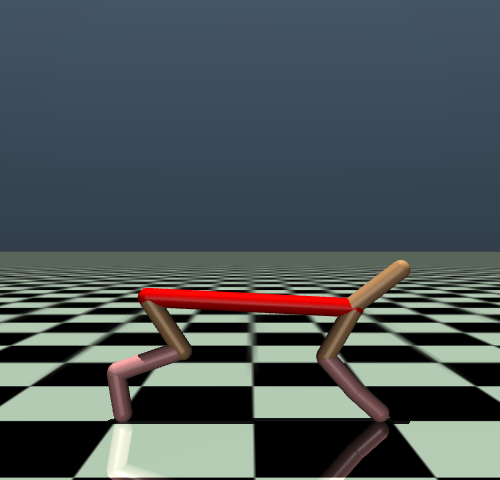}
    \includegraphics[width=0.19\linewidth]{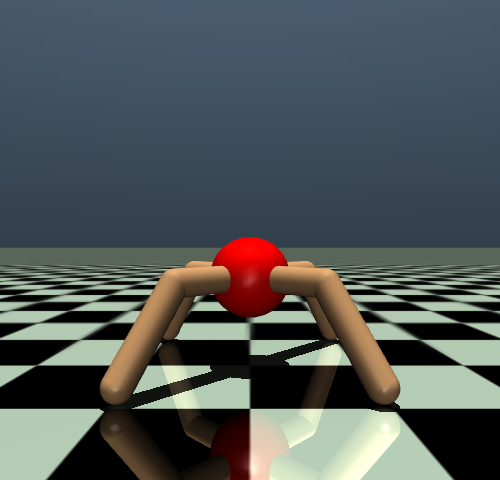}
    \includegraphics[width=0.19\linewidth]{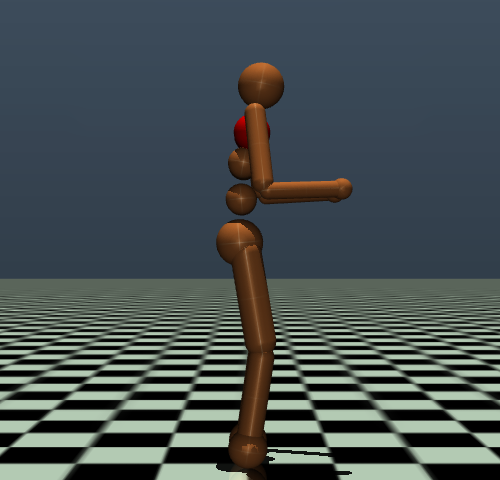}
    \caption{\textbf{MuJoCo environments with mass–scaled body parts.}  
    We vary a global \(\kappa\) on only the red-highlighted links in
    each agent (e.g., torso for
    \texttt{HalfCheetah}). \texttt{CartPole} (not shown in the figure) scales pole mass, pole length, and cart mass simultaneously.}
    \label{fig:enter-label}
\end{figure}

\vspace{3pt}\noindent
\textbf{Task and context specifications.} All environments are standard MuJoCo tasks whose dynamics depend on
\emph{hidden} physical parameters that once drawn, remain fixed for the episode.  Depending on the baseline, contexts are drawn from either a
\textit{continuous} sampler or a \textit{discrete} curriculum during
training, and performance is evaluated on a wider
out-of-distribution (OOD) range.  
Unless otherwise stated, a global \emph{scaling} factor~\(\kappa\)
multiplies every context-relevant physical parameter:
\[
  \kappa \in 
  \mathcal C_{\text{train}}=[0.75,\,2.00], 
  \qquad 
  \kappa_{\text{test}} \in \mathcal C_{\text{test}} = [0.50,\,2.50].
\]
During training, each algorithm samples \(\kappa\) according to its own
rule (uniform, curriculum, etc.). At evaluation time, contexts are drawn \emph{uniformly} from the full continuous range, and we report the average return to ensure a fair comparison:

\vspace{4pt}
\begin{enumerate}
\item \texttt{CartPole} \;(3-D context).  
      The base parameter vector
      \(\bigl(m_{\text{pole}},\,\ell_{\text{pole}},\,m_{\text{cart}}\bigr)\)
      is multiplied by a scalar \(\kappa \sim \mathcal C_{\text{train}}\). 
      The OOD set uses the same scaling, but with
      \(\kappa \sim \mathcal C_{\text{test}}\).

\item \texttt{Hopper}, \texttt{Walker2d}, \texttt{HalfCheetah},
      \texttt{Ant}, \texttt{Humanoid} \;(1-D context).  
      A single global mass scale \(\kappa\) is drawn per episode:
      \[
        \kappa \sim
        \begin{cases}
          [0.75,\,2.00], & \textit{continuous train},\\[3pt]
          \{0.75,\,1.08,\,1.42,\,1.75\}, & \textit{discrete train},\\[3pt]
          [0.50,\,2.50], & \textit{test}.
        \end{cases}
      \]
\end{enumerate}

\begin{figure}[ht]
\begin{center}
\centerline{\includegraphics[width=0.95\textwidth]{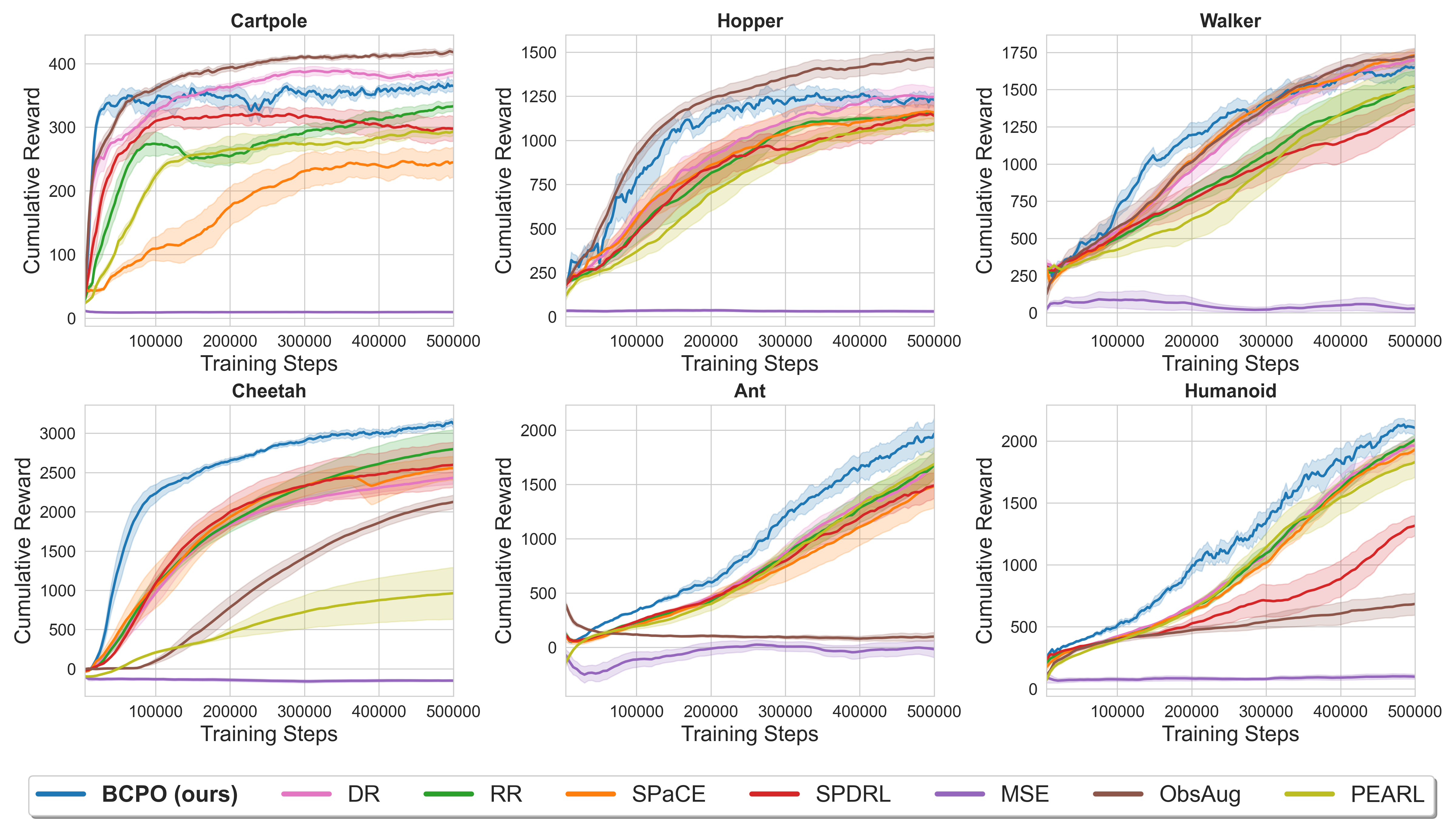}}
\caption{\textbf{Evaluation Performance.} Returns vs. Training steps. All methods are evaluated on testing contexts every 5000 training steps. Results averaged over 5 seeds.}
\label{fig:eval}
\end{center}
\end{figure}

\subsection{Main results}\label{subsec:main results}
Figure~\ref{fig:eval} reports evaluation learning curves, and
Table~\ref{tab:testing-results} summarizes the \emph{asymptotic} test
returns at ~$5\times10^{5}$ steps. Complete training results are in
the Appendix (see Figure~\ref{fig:training_result_plot} and Table~\ref{tab:training-results}).

\paragraph{Performance \& sample efficiency.}
BCPO dominates five of the six benchmarks and a close second on the
two easiest tasks.  Using a common ``expert'' threshold 80\% of BCPO’s final score as the success target, BCPO reaches it in  
\textbf{30 k / 60 k / 90 k / 70 k / 110 k / 140 k} environment steps on \texttt{CartPole, Hopper, Walker2d, Cheetah, Ant} and \texttt{Humanoid}, respectively.  In
contrast, Domain Randomization requires 2-3\(\times\)more data. Curriculum methods (\textbf{SPaCE, SPDRL}) are another
30–50 k steps slower.

\begin{table}[ht]
  \centering
    \caption{Average episodic return (mean $\pm$ coefficient of variation) on
           six MuJoCo tasks, all evaluated under the same continuous–context
           protocol (see Sec.~\ref{subsec:benchmarks}).\;
           Higher is better.
           The best score in each column is \textbf{bold} and \best{dark-shaded};
           the 2\textsuperscript{nd} and 3\textsuperscript{rd} are
           \second{medium} and \third{light} shaded, respectively.}
  \label{tab:testing-results}
  \resizebox{\textwidth}{!}{
  \begin{tabular}{@{}lcccccc@{}}
    \toprule
    \textbf{Method} & CartPole & Hopper & Walker & Cheetah & Ant & Humanoid \\
    \midrule
    BCPO(\textbf{ours}) &
      \third{$365.18\!\pm\!0.11$} &
      \third{$1233.72\!\pm\!0.11$} &
      $1652.02\!\pm\!0.09$ &
      \best{$\mathbf{3129.05\!\pm\!0.07}$} &
      \best{$\mathbf{1964.57\!\pm\!0.17}$} &
      \best{$\mathbf{2105.60\!\pm\!0.08}$} \\
    \midrule
    DR &
      \second{$383.56\!\pm\!0.09$} &
      \second{$1266.34\!\pm\!0.15$} &
      \third{$1715.50\!\pm\!0.10$} &
      $2449.73\!\pm\!0.07$ &
      $1685.28\!\pm\!0.31$ &
      \third{$2018.49\!\pm\!0.13$} \\

    RR &
      $336.83\!\pm\!0.11$ &
      $1153.56\!\pm\!0.10$ &
      $1544.74\!\pm\!0.17$ &
      \second{$2832.10\!\pm\!0.18$} &
      \third{$1716.06\!\pm\!0.22$} &
      \second{$2080.35\!\pm\!0.11$} \\

    SPaCE &
      $243.05\!\pm\!0.28$ &
      $1170.31\!\pm\!0.10$ &
      \best{$\mathbf{1762.86\!\pm\!0.10}$} &
      $2598.73\!\pm\!0.11$ &
      $1530.62\!\pm\!0.29$ &
      $1998.17\!\pm\!0.13$ \\

    SPDRL &
      $293.64\!\pm\!0.21$ &
      $1153.27\!\pm\!0.24$ &
      $1389.91\!\pm\!0.18$ &
      \third{$2614.98\!\pm\!0.24$} &
      $1539.81\!\pm\!0.26$ &
      $1386.59\!\pm\!0.25$ \\

    MSE &
       $9.62\!\pm\!0.14$ &
      $31.11\!\pm\!0.52$ &
      $31.80\!\pm\!2.86$ &
     $-149.57\!\pm\!0.20$ &
      $-3.96\!\pm\!43.07$ &
      $101.24\!\pm\!0.51$ \\

    ObsAug &
      \best{$\mathbf{418.77\!\pm\!0.07}$} &
      \best{$\mathbf{1477.30\!\pm\!0.11}$} &
      \second{$1742.24\!\pm\!0.10$} &
      $2178.13\!\pm\!0.09$ &
      $102.03\!\pm\!0.96$ &
      $698.31\!\pm\!0.28$ \\

    PEARL &
      $295.90\!\pm\!0.14$ &
      $1103.79\!\pm\!0.12$ &
      $1558.95\!\pm\!0.19$ &
      $975.82\!\pm\!0.69$ &
      \second{$1740.66\!\pm\!0.23$} &
      $1875.71\!\pm\!0.15$ \\
    \bottomrule
  \end{tabular}
  }
  \vspace{-0.5em}
  \footnotesize Results averaged over five random seeds.
\end{table}

\paragraph{Method-specific observations.} \textbf{ObsAug} excels on \texttt{CartPole/Hopper} where a 10-step history adds \(<\)50 dimensions, but lags \textbf{BCPO} by 600–900 reward on \texttt{Cheetah/Ant} as the augmented state becomes harder to fit. \textbf{DR} scales reasonably with DoF yet never closes the last 10–20 \% return gap, showing that an implicit policy cannot exploit hidden context even with broad exposure. \textbf{SPaCE/SPDRL} suffers from ``curriculum forgetting '' (further discussed in~\citep{eimer2021self}): once the context sampling distribution shifts, value targets drift and early-stage skills must be reinforced, which explains its limited generality over a wide context set. \textbf{PEARL} ranks mid-pack among all tasks except on \texttt{Cheetah.} \textbf{RR}’s fixed context set cycle keeps the replay buffer perfectly balanced, giving stable gradients for SAC and solid interpolation in the 1-D mass-scale setting, but its coarse grid still caps performance once task DoF or context complexity increases.

\paragraph{Stress test.} To study the limits of the dual observation–policy architecture, we sweep the scaling factor well beyond the train regime and record evaluation returns in Figure~\ref{fig:stress-test}. Across all six tasks, \textbf{BCPO} degrades gracefully: reward declines smoothly rather than collapsing at the train-OOD regime boundary, indicating that the encoder continues to infer even for unseen contexts. Failure modes are rather task-specific: for example, with its three coupled
parameters, \texttt{CartPole} loses balance fastest when the pole is
ultra-light. Interestingly, \texttt{HalfCheetah} forms a textbook bell curve around $\kappa \approx 1.4$. The smooth degradation confirms our information residual analysis: once the encoder achieves observation sufficiency, any remaining sub-optimality stems from classical dynamics and control constraints, not from a failure of the learned representation.

\begin{figure}[ht]
\begin{center}
\centerline{\includegraphics[width=0.97\textwidth]{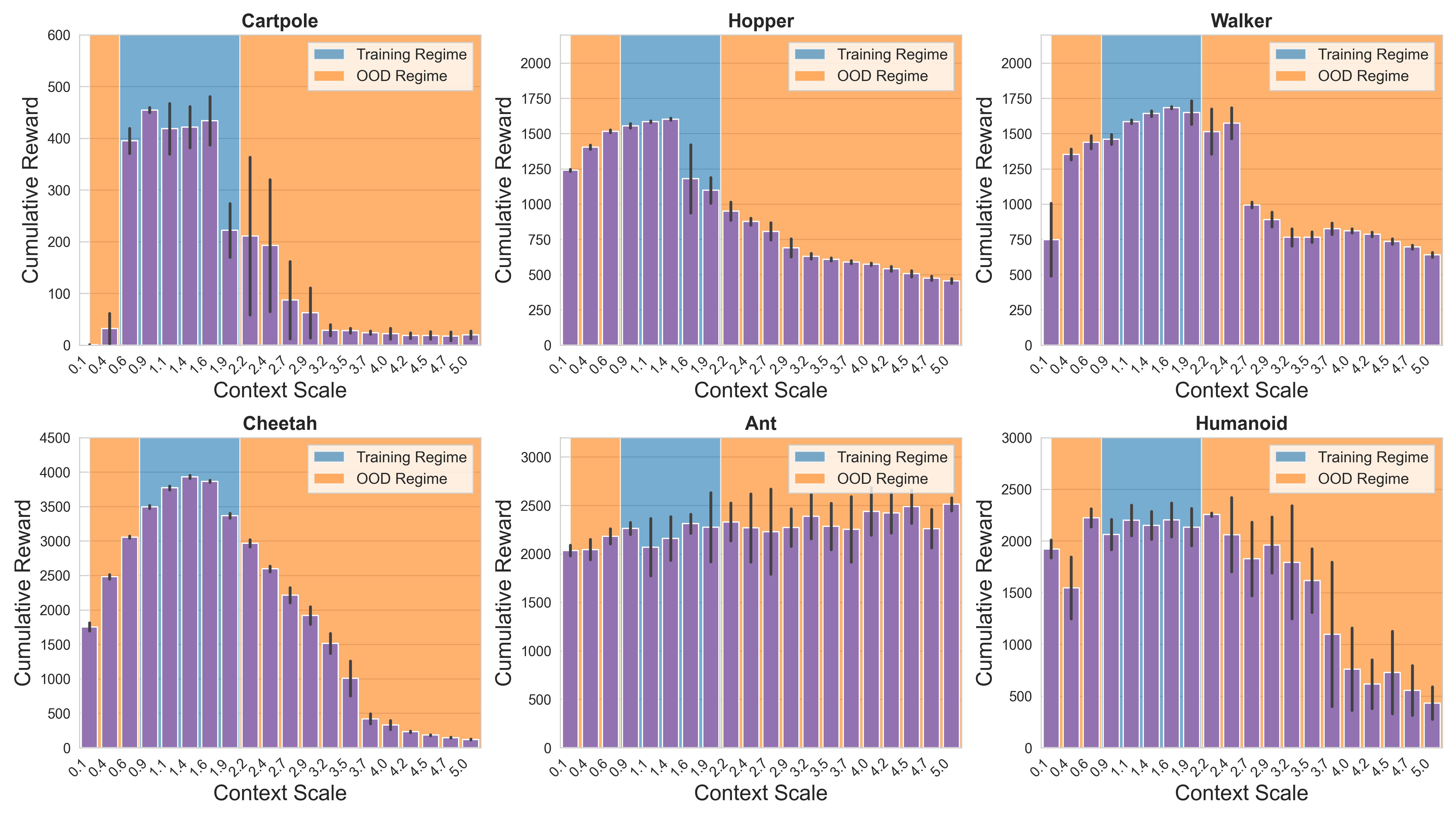}}
\caption{\textbf{Stress test.} Blue shading indicates the \textcolor{Cyan}{training interval} \([0.75,2.0]\), and orange shading the \textcolor{Peach}{OOD regime} [0.1, 5.0].}
\label{fig:stress-test}
\end{center}
\end{figure}

\paragraph{IB Weight ablation.} Figure~\ref{fig:beta_ablation} contrasts three \emph{static} choices of \((0.1,0.5,1.0)\) with three \emph{monotonic} schedules that anneal from a low exploration-friendly value to a high compression one. A small initial $\beta$ (light bottleneck) lets the encoder pass most features, so the policy learns quickly. A larger final $\beta$ is essential for pruning redundancy for robustness and stabilizing long-horizon performance. 

From the \emph{control--sufficiency} standpoint (see Figure~\ref{fig:beta_ablation}), this bottleneck tightening is essential. Early, a context-rich latent code $Z$ enables aggressive exploration, but it retains the value targets at high variance. As clusters sharpen, redundant features are pruned, the critic estimates
stabilize, and the actor can specialize contextual behavior without overfitting. A \emph{fixed low} $\beta$ would keep superfluous details and cap final
return; a \emph{fixed high} $\beta$ would compress too soon and slow the exploration. The annealed schedule thus balances both extremes: expressive codes first, minimal codes later.

From an \emph{observation-sufficiency} standpoint, the latent code $Z$ evolves exactly as predicted by our information residual evidence (see Figure~\ref{fig:vis}). At \textbf{1\,\%} of training, the embeddings are tangled, indicating the encoder has not yet separated the contexts. By \textbf{2--15\,\%}, the contrastive InfoNCE term has driven inter-cluster formation with KL term shrinking the intra-cluster covariances, which indicates that $I(C;Z)$ is rapidly approaching $I(C;O)$. When $\beta$ begins to rise, the compression term dominates: the bottleneck squeezes away residual noise and collapses intra-cluster variance into almost isotropic spheres. The final state has six compact well-separated Gaussians, which is a direct low-dimensional evidence that the encoder now carries the minimal bits required for observation sufficiency.

\begin{figure}[ht]
    \centering
    \subfigure[\textbf{$\beta$ schedules used in the ablation.}\label{fig:beta_sched}]{
        \includegraphics[width=0.46\linewidth]{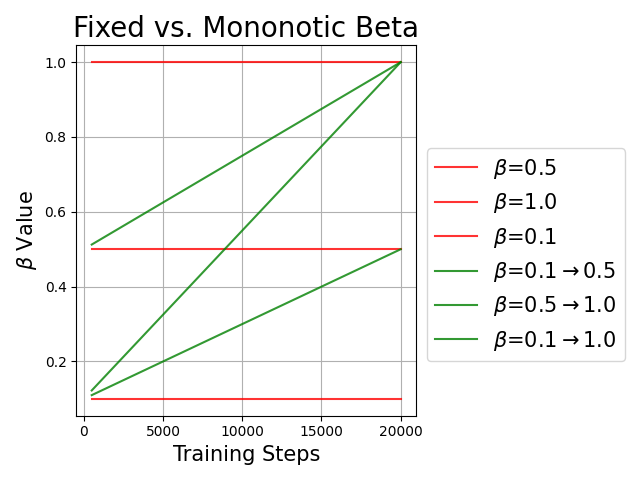}}
    \hfill
    \subfigure[\textbf{Effect of $\beta$ on evaluation return.}\label{fig:beta_return}]{
        \includegraphics[width=0.46\linewidth]{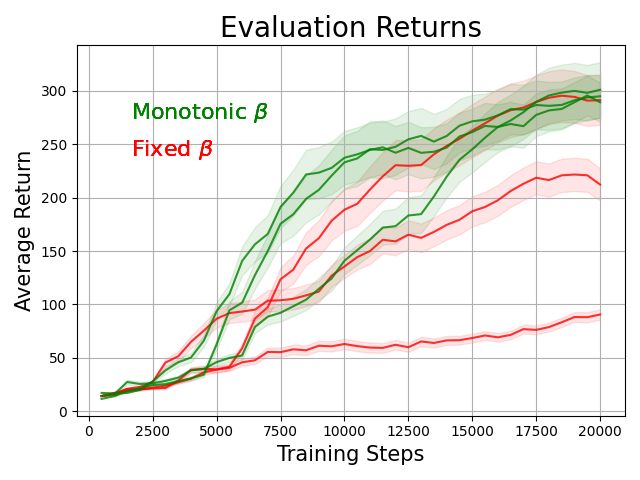}}
    \caption{\textbf{Ablation on the IB weight $\beta$.}  
    \textbf{(a)}~Fixed settings (red) keep $\beta$ constant at \{0.1, 0.5, 1.0\}; 
    monotonic schedules (green) linearly anneal $\beta$ from a low exploration-friendly
    value to a high compression-weighted one.  
    \textbf{(b)}~A small initial $\beta$ accelerates early learning, while a
    larger final $\beta$ is crucial for asymptotic performance.  The best
    trade-off is achieved by an annealed schedule: rich features first for exploration, minimal code later for robustness.  Shaded
    regions are 95\% CIs over three seeds.}
    \label{fig:beta_ablation}
\end{figure}

\begin{figure}[ht]
    \centering
    \subfigure[Progress: 1\%]{%
        \includegraphics[width=0.19\textwidth]{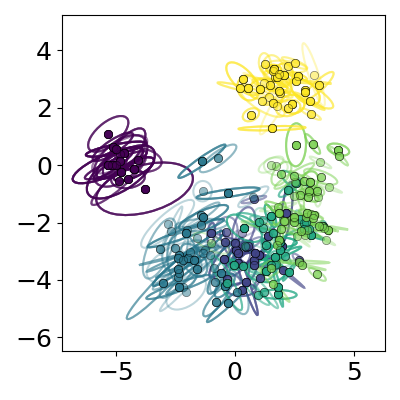}%
    }
    \subfigure[2\%]{%
        \includegraphics[width=0.19\textwidth]{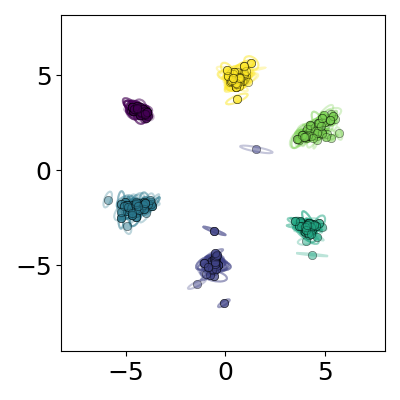}%
    }
    \subfigure[15\%]{%
        \includegraphics[width=0.19\textwidth]{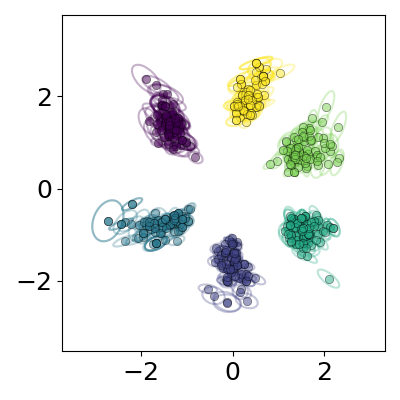}%
    }
    \subfigure[25\%]{%
        \includegraphics[width=0.19\textwidth]{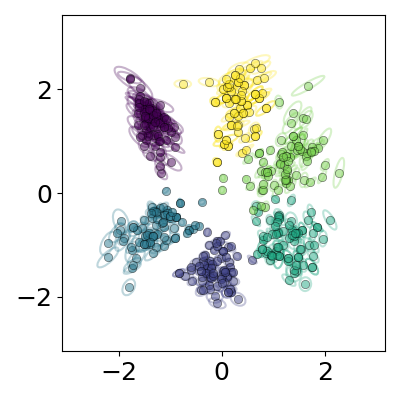}%
    }
    \subfigure[85\%]{%
        \includegraphics[width=0.19\textwidth]{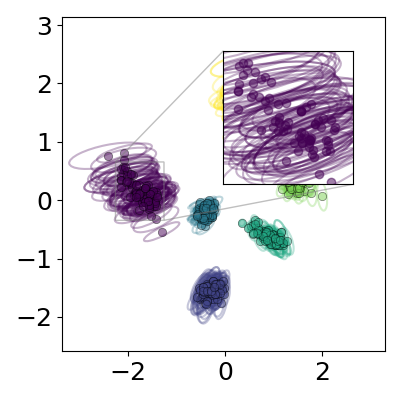}%
    }
    \caption{\textbf{Embedding Visualization.} Evolution of the latent context distribution (dim=2) during \text{BCPO} training on the Cartpole Env (6 different rod lengths). Different colors highlight different contexts. The embedding co-evolves 
    with the policy. The encoder \( \mathcal{N}(z|\mu_{\phi}(o), \Sigma_{\phi}(o)) \) is used 
    with a linearly increasing \(\beta\)-schedule ($0.0001\nearrow0.1$). We plot each embedding’s mean and covariance 
    (shown as ellipses with 95\% confidence intervals). As training progresses, the encoder progressively 
    distills context-relevant features from collected data, forming distinct context clusters and compressing away 
    redundancies (note the scale changes on the axes and distinct \emph{Gaussian-like} clusters at 85\% progress).%
    }
    \label{fig:vis}
\end{figure}

\subsection{Discussion}
\paragraph{Failure modes analysis.} Our error decomposition of information residual pinpoints why some explicit-context baselines succeed in certain regimes yet stall in others.

\textbf{MSE} treats context inference as a \emph{point–estimate}
problem: given a history window \(O\) it emits a single, fully
\textbf{deterministic} guess \(\hat c\) and feeds \((s,\hat c)\) to the
policy.  This is equivalent to assuming the posterior
\(q(c|O)=\delta(c-\hat c)\), i.e.\ the context is noise-free once
\(O\) is observed.  Under our information residual decomposition, such an
assumption declares the \textsc{Encoder Gap} closed as soon as the
\textbf{training} MSE is small.  At test time, however, the same predictor is systematically biased; the policy, co-adapted to \(\hat c\),
receives a mismatched pair \((s,\hat c)\).  State trajectories drift
outside the predictors’ training support, errors compound, and
performance collapses. This reveals that the zero gap is a by-product of the misspecified point-estimate assumption rather than the observation sufficiency.

\textbf{ObsAug} can be read as a ``\textbf{BCPO} without the bottleneck'': it simply feeds the entire window \(O\) to the policy.  
When the state is low-dimensional (e.g., \texttt{CartPole/Hopper}), this shortcut already makes the \textsc{Processing Gap} negligible and nothing
catastrophic happens. But as the task DoF rises, the augmented state becomes high-dimensional: the critic must now use a high-variance value function from a signal filled with redundant context bits. In our decomposition, this redundancy \emph{is} the \textsc{Encoder Gap}: because \textbf{ObsAug} never shrinks it, every replayed batch re-injects noise into the update and magnifies the \textsc{Replay Gap}. Consequently, both sample efficiency and asymptotic return fall behind
\textbf{BCPO} precisely on the high-DoF tasks where explicit information–bottleneck compression matters most.

\textbf{PEARL} appears consistently mid-pack because two opposing
errors nearly cancel. First, \citeauthor{rakelly2019efficient} mitigates replay bias with a ``recent buffer'' heuristic, so the \textsc{Replay Gap} term is small, a fact predicted by Lem.~3.3. Second, \textbf{PEARL} makes a strong \emph{factorization} assumption inside its
data window:
\[
q_{\phi}\!\bigl(z\mid o_{1:k}\bigr)
\;\propto\;
\prod_{i=1}^{k} \Phi_{\phi}\!\bigl(z\mid o_i\bigr),
\]
where $\Phi$ is a Gaussian factor that results in a Gaussian posterior. This collapses temporal correlations and leaves a residual
\textsc{Processing Gap} that \textbf{BCPO} avoids.  
Moreover, the meta-RL objective maximizes \emph{average} return over all
contexts, whereas \textbf{BCPO} learns distinct latent-conditioned behaviors. Thus, \textbf{PEARL} rarely tops the chart but seldom crashes either.

In short, the information residual perspective tells that an explicit context robust policy must keep the \textsc{Processing}, \textsc{Encoder}, \textsc{Replay} gaps simultaneously small. Under point-estimate assumption, MSE reopens the \textsc{Encoder Gap} during testing; \textbf{ObsAug} ignores \textsc{Encoder Gap} and amplifies the \textsc{Replay Gap}; \textbf{PEARL} leaves a residual in the \textsc{Processing Gap} through its factorization assumption. \textbf{BCPO} is the only method that simultaneously and explicitly tackles all, which is precisely why it is fast, stable, and robust across the full spectrum of tasks.

\paragraph{Implicit vs. explicit context policy.} From an algorithm design perspective, our study reaffirms a ``no–free–lunch'' principle for contextual RL. \emph{Implicit} approaches (\textbf{DR, RR, self-paced curriculum}) pay \textbf{zero} on representation/observation and work well when the context manifold is simple, but they hit a performance ceiling once task DoF or context entropy increases: the single policy must
simultaneously \textit{infer} and \textit{control}, which wastes its function approximation budget (e.g., neural networks).  

\emph{Explicit} methods introduce a dedicated context channel, but that channel comes with extra parameters, inference overhead, and the risk of new failure modes. \textbf{BCPO} shows that this complexity pays off \emph{only} when the channel is carefully designed: the data-processing window is informative yet compact, posterior uncertainty is explicitly modeled, and replay drift is controlled. Remove any one of these safeguards, and the added machinery yields diminishing returns, or in the worst case, increases computation without narrowing the performance gap.

\section{Related works}\label{sec:related}

\paragraph{Contextual \& meta RL.}
\emph{Contextual RL} learns a single policy that generalizes across tasks with contextual inputs~\citep{hallak2015contextual,benjaminscontextualize,kupcsik2013data, oord2018representation, lee2020context}. Context-based meta RL methods (e.g.,~\citealp{gao2024context,pong2022offline}) and PEARL in particular~\citep{rakelly2019efficient} infer a latent task embedding from recent/behavioral experience and condition the
policy on that embedding. Pearl explicitly uses replay heuristics to
train off-policy critics.  Conservative Q-learning pushes this replay idea further by penalizing actions outside the behavior support~\citep{kumar2020conservative}. 

\paragraph{Randomization \& curriculum.}
Domain randomization broadens the training context exposure by sampling
diverse environment parameters~\citep{andrychowicz2020learning, margolisyang2022rapid}, and curriculum schedulers progressively expand the task set to stabilize
learning~\citep{mehta2020active,eimer2021self, klink2020self}.

\paragraph{Representation learning, IB, \& information structure.} A growing body of work studies \emph{representation–oriented} RL through the lens of predictive or world models. Many employ
information bottleneck regularization to encourage latent codes that are compact yet decision–relevant, e.g. DeepMDP~\citep{gelada2019deepmdp}, Dreamer~\citep{hafner2023mastering}, stochastic latent actor-critic~\citep{lee2020stochastic}, and many more. Closer to our \emph{control–sufficiency} viewpoint, \citeauthor{huang2022action} explicitly optimizes latent variables that preserve the action–value ordering.

\paragraph{Historical remark.}
For completeness, we relate our framework and algorithm to classical results in stochastic control. The modular structure exploited by \textbf{BCPO}—and, more broadly, by many context–aware RL methods—echoes the classical separation principle: under suitable structural assumptions, the estimation and control sub-problems can be solved independently and then combined without loss of optimality~\citep{wonham1968separation,georgiou2013separation,aastrom2012introduction,fleming2012deterministic}.  
The canonical illustration is the pairing of the Kalman filter and the LQR controller~\citep{kalman1960new}.  
Inspired by the same idea, \textbf{BCPO} explicitly optimizes for a \emph{separable} representation–policy learning problem to promote algorithmic simplicity, sample efficiency, and stronger generalization in data-driven RL settings.

\section{Conclusion and Future Work}\label{sec:future}
This paper presents a unified inference–control perspective on context-based RL with latent, episode-level variations. We formalized the gap between \emph{observation sufficiency}(knowing the latent context) and \emph{control sufficiency} (acting optimally) through an \emph{information residual}, derived a tractable ELBO-style objective whose two terms cleanly separate representation learning (inference) from policy improvement (control), and instantiated the theory in \textbf{BCPO}.

The results presented here aim to serve both theory and practice. For theorists, the framework offers a unified information-theoretic lens that subsumes earlier, disparate notions—such as representation quality and value-function mismatch—to study the inference–control structure in RL. For practitioners, the residual’s explicit decomposition turns abstract information-theoretical terms into concrete diagnostics. Thus, the framework not only clarifies \emph{what} matters in context-based RL, but also \emph{how} to improve existing algorithms through the \textbf{BCPO} procedure.

The present study assumes fixed, episode-level contexts and relies on variational MI estimators.  This leaves several natural extensions. Extending the sufficiency theory to non-stationary and within-episode context drift is non-trivial and may require new definitions. Combining \textbf{BCPO} with latent-variable dynamics/world models could yield planning-aware bottlenecks and a sharper inference–control separation. Additionally, embedding explicit risk or information-budget constraints into \textbf{BCPO} is a promising route toward safe, real-world deployment.

\clearpage
\section*{Acknowledgments}
This work was supported in part by NASA ULI (80NSSC22M0070), Air Force Office of Scientific Research (FA9550-21-1-0411), NSF CPS (2311085), NSF RI (2133656), NSF CMMI (2135925), NASA under the Cooperative Agreement 80NSSC20M0229, and NSF SLES (2331878). Marco Caccamo was supported by an Alexander von Humboldt Professorship endowed by the German Federal Ministry of Education and Research.

\appendix
\section{Variational Information Bottleneck}\label{app:vib}
In this section, we provide a detailed exposition of the \emph{Variational Information Bottleneck} (VIB) framework. Recall the information bottleneck objective~\eqref{eq:ib} in its parametric form:
\begin{equation}
    \min_{\phi}\ \mathcal{L}_{\mathrm{IB}}(\phi) \;=\; -\I_{\phi}(Z; C) \;+\; \beta\ \I_{\phi}(Z;O) ,
\end{equation}
where $\beta > 0$ regulates the trade-off between \emph{relevance} and \emph{compression}. Because direct computation of these mutual information (MI) terms is generally intractable, we adopt variational bounds to derive a practical training objective.

\paragraph{Upper-bounding \(\I_{\phi}(Z; O)\).} 
{
When the conditional distribution $p(z|o)$ is given (i.e., an encoder $p_{\phi}(z|o)$ parametrized by $\phi$), a tractable variational upper bound can be built by introducing variational approximation $r(z)$ to the unknown $p(z)$~\citep{barber2004algorithm, poole2019variational}:
\begin{align}
    \I_{\phi}(Z; O) &\;=\; \mathbb{E}_{p(z,o)} \Bigl[\log \frac{p_{\phi}(z|o)}{p(z)}\Bigr] \notag \\
    &\;=\; \mathbb{E}_{p(z,o)} \Bigl[\log \frac{p_{\phi}(z|o) r(z)}{p(z)r(z)}\Bigr] \notag \\
    &\;=\; \mathbb{E}_{p(z,o)} \Bigl[\log \frac{p_{\phi}(z|o)}{r(z)}\Bigr] - \underbrace{\KL \bigl(p(z)\| r(z)\bigr)}_{\ge 0} \notag \\
    &\;\leq\; \mathbb{E}_{p(o)} \Bigl[ \KL\bigl( p_{\phi}(z|o) \| r(z) \bigr) \notag \Bigr] \\
    &\;\triangleq\; \I_{\text{MIN}}(\phi).
\end{align}
Hence, minimizing $\I_{\text{MIN}}(\phi)$ \emph{upper bounds} the compression term $I_{\phi}(Z;O)$.

In practice, the conditional (encoding) distribution $p_{\phi}(z|o)$ is often chosen to be a Gaussian $\mathcal{N}\!\bigl(z|\mu_{\phi}(o), \Sigma_{\phi}(o)\bigr)$, with neural network parametrization $\mu_{\phi}$ and $\Sigma_{\phi}$. In this case, the KL term admits a closed-form expression and can be optimized via the standard re-parameterization trick~\citep{kingma2013auto}.
}

\paragraph{Lower-bounding \(\I_{\phi}(Z; C)\).} 
A variety of variational lower bounds \citep{barber2004algorithm, oord2018representation, poole2019variational} and mutual-information estimators \citep{belghazi2018mine} can be categorized according to whether they explicitly model the conditional distributions $p(c|z)$(or $p(z|c)$). We broadly group these approaches into \emph{explicit} (supervised) and \emph{implicit} (unsupervised or semi-supervised) settings:

\begin{itemize}
        \item \textbf{Explicit (Supervised).} 
        If the context $C$ is observed as labels (e.g., one-hot label, or trajectory index), one can introduce a variational \emph{decoder} $q_{\phi}(c|z)$ to recover the labels from $Z$. In this scenario, we can invoke the classical \emph{Barber--Agakov} bound \citep{barber2004algorithm} on mutual information, which follows from the non-negativity of the KL divergence:
        \begin{equation}
        \label{eq:BA_bound}
            I_{\phi}(Z;C) 
            \;\ge\; 
            \underbrace{
              \mathbb{E}_{p(z,c)}
                \bigl[\log q_{\phi}(c|z)\bigr]
            }_{\displaystyle \text{label reconstruction}}
            \;+\;
            \underbrace{H(C)}_{\displaystyle \text{label entropy}}
            \;\triangleq\; \I^e_{\text{MAX}}(\phi)
        \end{equation}
        The term \(H(C)\) is independent of the decoder and thus cannot be optimized with respect to 
        \(q_{\phi}\). Consequently, the optimization boils down to maximizing
        \begin{equation*}
          \mathbb{E}_{p(z,c)}
                \bigl[\log q_{\phi}(c|z)\bigr]
          \;=\;
          \mathbb{E}_{p(z)}
            \Bigl[ \underbrace{\mathbb{E}_{p(c|z)} \bigl[
              \log q_{\phi}(c|z) \bigr]}_{\displaystyle -H(p,q_{\phi})}
            \Bigr],
        \end{equation*}
        which corresponds to a standard cross-entropy (label reconstruction) objective. Combining this 
        variational bound \(\I_{\mathrm{MAX}}^{e}(\phi)\) with \(\I_{\mathrm{MIN}}(\phi)\) yields a 
        tractable information bottleneck objective, essentially the VIB formulation from 
        ~\citep{alemi2016deep}.
            
        \item \textbf{Implicit (Semi-supervised).} In many practical scenarios, \(C\) may be partially labeled or unobservable. Nonetheless, observations \(\boldsymbol{o}_{1:K} \sim p(o_{1:K}|\boldsymbol{c}_i)\) from an underlying \(c_i\) share similar structures, and thus the corresponding representations \(\boldsymbol{z}_{1:K} \sim p(z_{1:K}|o_{1:K},\boldsymbol{c}_i)\) reflect this shared structure. By grouping similar representations while separating dissimilar ones, we effectively maximize the mutual information between \(Z\) and \(C\)  by clustering. This can be achieved by optimizing InfoNCE, a contrastive loss that encourages the model to distinguish between positive (similar) and negative (dissimilar) pairs \cite{oord2018representation}. Here, \(\boldsymbol{c}_i\) can be used as \emph{pseudo-labels} to facilitate contrastive sampling.

         Concretely, suppose we draw a set of \(Z^K = \{\mathbf{z}^{\text{pos}}, \mathbf{z}^{\text{neg}}_1, \dots, \mathbf{z}^{\text{neg}}_{K-1}\}\), where \(\mathbf{z}^{\text{pos}}\) is an embedding sample from $p(z|o, \boldsymbol{c}_i)$ (the positive), and \(\{\mathbf{z}^{\text{neg}}_j\}\) are embedding samples from $p(z|o,\boldsymbol{c}\neq \boldsymbol{c}_i)$. The InfoNCE objective is:
        \begin{equation*}
            \label{eq:infonce}
            \mathcal{L}_{\text{InfoNCE}}(\psi) 
            = 
            -\mathbb{E}_{Z^K}\Biggl[
                \log \frac{f_{\psi}\bigl(\mathbf{z}^{\text{pos}}, \mathbf{z}_i\bigr)}
                          {\sum_{z_j \in Z^K} f_{\psi}\bigl(\mathbf{z}^{\text{pos}}, \mathbf{z}_j\bigr)}
            \Biggr],
        \end{equation*}
        where \(f_{\psi}(\cdot,\cdot)\) is a learnable critic function (e.g., a bilinear function 
        \(f_{\psi}(z,z')=z^T W_{\psi}z'\)) that scores the similarity between samples. 
        As a categorical cross-entropy loss, \(\mathcal{L}_{\text{InfoNCE}}\) effectively classifies 
        the positive sample \(\mathbf{z}^{\text{pos}}\) against \(K-1\) negatives. Combining the InfoNCE loss with the variational bound \(\I_{\mathrm{MAX}}^{e}(\phi)\) 
        results in a tractable \textbf{implicit} VIB objective. 
        
        Note that while the explicit VIB objective is strictly a variational lower bound of the original IB 
        objective~\cite{alemi2016deep}, this implicit IB objective is rather an approximation. Although 
        the InfoNCE loss is known to be a variational lower bound on mutual information~\cite{poole2019variational}, 
        we use it here primarily as a contrastive loss to softly partition the latent embeddings \(Z\) 
        into a mixture of distributions (i.e., cluster embeddings from the same context). Once \(p(z)\) 
        evolves into a mixture distribution, it aligns with the solution structure of IB (see the self-consistent equations of IB solution in~\cite{tishby2000information, slonim1999agglomerative}).
    \end{itemize}
By combining these bounds, we obtain a tractable \emph{Variational IB Loss}. While the supervised approach often offers computational efficiency, the implicit one yields more robust and generalizable representations, particularly in settings where labeled data is limited but there is a large pool of observations (as is often the case in RL).

\section{Full Experimental Results}\label{appendix:exp_setup}
In this section, we provide implementation details and present the complete experimental results. The source code is available at: \url{https://github.com/HP-CAO/physics-rl}.

\subsection{Environment Setup}
We summarize the key parameters for the five environments in Table~\ref{tab:env_setup}. For further details, including task definitions, reward function design, and termination conditions, please refer to~\citep{cao2022cloud} for the Cartpole environment and~\citep{mujoco} for the five MuJoCo environments.

\begin{figure}[h]
\begin{center}
\centerline{\includegraphics[width=0.98\textwidth]{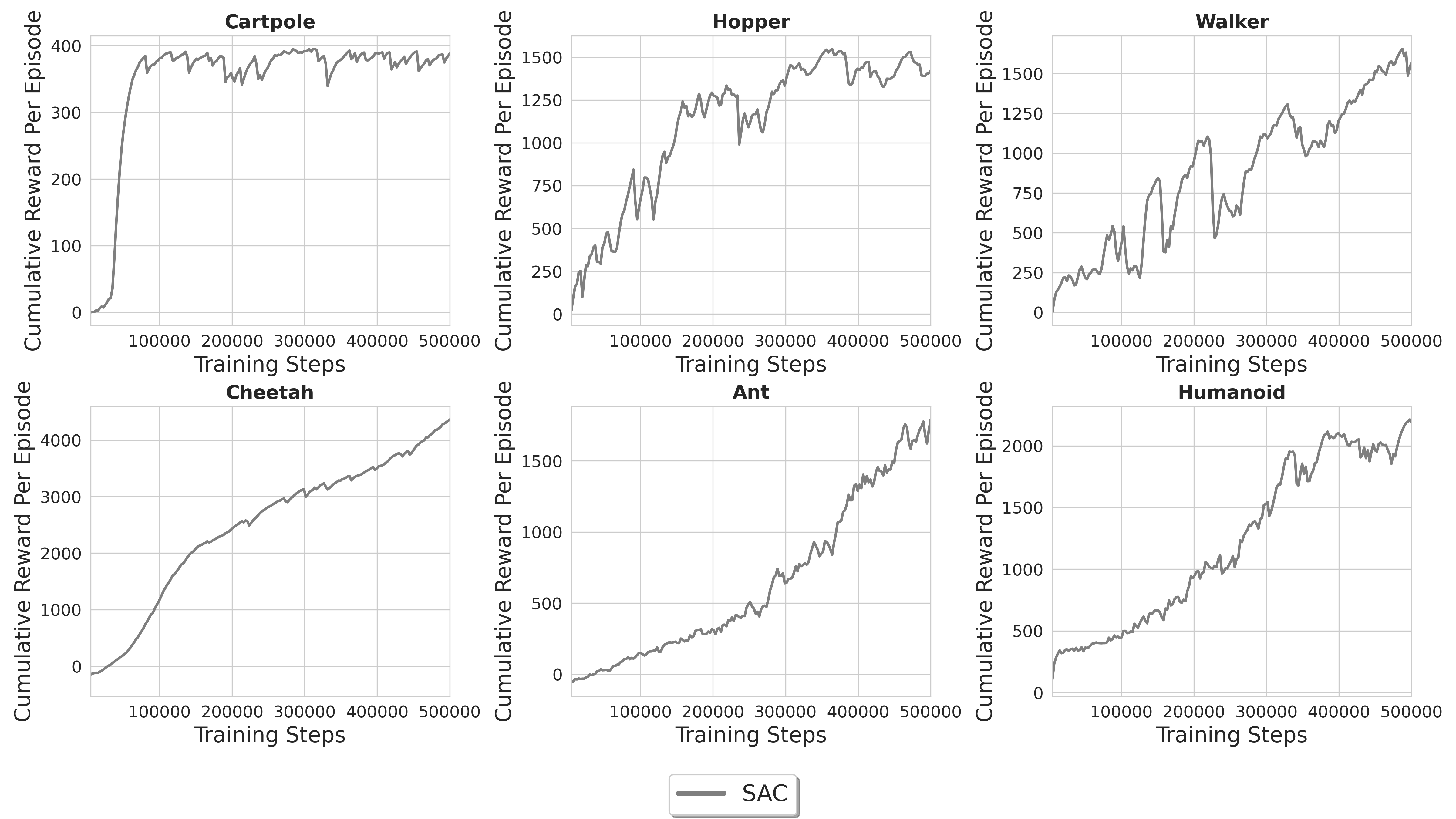}}
\caption{\textbf{Baseline Training Performance.} Returns vs. Training steps. Trained on default environments.}
\label{fig:sac_baseline}
\end{center}
\end{figure}

\begin{table}[ht]
    \centering
    \renewcommand{\arraystretch}{1.2}
    \caption{Observation and Action Space Dimensions for MuJoCo Environments}
    \label{tab:env_setup}
    \begin{tabular}{lcccc}
        \hline
        \textbf{Environment} & \textbf{State Dim.} & \textbf{Action Dim.} & \textbf{C. Dim. } & \textbf{Policy Hidden Layers} \\
        \hline
        Cartpole & 5 & 1 & 3 & [256, 256]  \\
        Hopper & 11 & 3 & 1 & [256, 256]  \\
        Walker2d & 17 & 6 & 1 & [256, 256] \\
        HalfCheetah & 17 & 6 & 1& [256, 256] \\
        Ant & 105 & 8 & 1 & [256, 256, 256]  \\
        Humanoid & 348 & 17 & 1 & [256, 256, 256] \\ 
        \hline
    \end{tabular}
\end{table}

\subsection{Baseline and Hyperparameters}
Our implementation of Soft Actor-Critic follows~\citep{pmlr-v80-haarnoja18b} and~\citep{fujimoto2018addressing}, with default parameters summarized in Table~\ref{tab:sac_hyperparams}. These parameters are shared across all baseline algorithms unless specified otherwise. The baseline training performance on the default environments is presented in Fig.~\ref{fig:sac_baseline}.

\begin{table}[ht]
    \centering
    \renewcommand{\arraystretch}{1.2}
    \caption{Parameter setting for Soft Actor-Critic (SAC)}
    \label{tab:sac_hyperparams}
    \begin{tabular}{l c}
        \hline
        \textbf{Hyperparameter} & \textbf{Value} \\
        \hline
        Discount factor ($\gamma$) & 0.99 \\
        Learning rate (actor, critic) & $3 \times 10^{-4}$ \\
        Optimizer & Adam \\
        Replay buffer size & $5*10^5$ \\
        Batch size & 256 \\
        Target smoothing coefficient ($\tau$) & 0.005 \\
        Entropy coefficient & 0.1 \\
        Target update interval & 1 \\
        Activation function & ReLU \\   
        Training steps per environment step & 1 \\
        Warm-up steps (without updates) & 3000 \\
        Total training steps & $5*10^5$ \\
        Maximum steps per episode & 500
        \\
        Evaluation period (steps) & 5000 \\
        Number of rollouts for evaluation & 50 \\
        \hline
    \end{tabular}
\end{table}

The parameters for the proposed \textbf{BCPO} algorithm are summarized in Table~\ref{tab:ib_parameter}. \textbf{MSE} and encoders of \textbf{PEARL} and \textbf{BCPO} are sharing the same backbone.

\begin{table}[h]
    \centering
    \renewcommand{\arraystretch}{1.2}
    \caption{BCPO Parameters}
    \label{tab:ib_parameter}
    \resizebox{\textwidth}{!}{
    \begin{tabular}{lcccccc}
        \hline
        \textbf{Parameter} & \textbf{Cartpole} & \textbf{Hopper} & \textbf{Walker2d} & \textbf{Cheetah} & \textbf{Ant} & \textbf{Humanoid} \\
        \hline
        Batch size  & 128 & 128 & 128 & 128 & 128 & 128 \\
        IB embedding dim & 2 & 4 & 4 & 4 & 15 & 30 \\
        IB encoder hidden units & [256, 256, 64] & [512, 512, 128] & [512, 512, 128] & [512, 512, 128] & [512, 512, 128] & [512, 512, 128] \\
        IB encoder activation  & GeLU & GeLU & GeLU & GeLU & GeLU  & GeLU \\
        IB encoder layer norm & True & True & True & True & True & True \\
        $\beta$ & 1.0e-04 $\nearrow$ 0.1 & 1.0e-04$\nearrow$ 0.1 & 1.0e-04$\nearrow$ 0.1 & 1.0e-04$\nearrow$ 0.1 & 1.0e-04 $\nearrow$ 0.1 & 1.0e-04 $\nearrow$ 0.1\\
        Number of context & 8 & 4 & 4 & 4 & 4 & 4 \\
        Number of negatives (InfoNCE) & 7 & 3 & 3 & 3 & 3 & 3\\
        Critic (InfoNCE) & Bilinear & Bilinear& Bilinear& Bilinear& Bilinear & Bilinear\\
        Prior $r(z)$ & \(\mathcal{N}(\mathbf{0}, \mathbf{I})\)  & \(\mathcal{N}(\mathbf{0}, \mathbf{I})\) & \(\mathcal{N}(\mathbf{0}, \mathbf{I})\) & \(\mathcal{N}(\mathbf{0}, \mathbf{I})\) & \(\mathcal{N}(\mathbf{0}, \mathbf{I})\) & \(\mathcal{N}(\mathbf{0}, \mathbf{I})\) \\
        \hline
    \end{tabular}
    }
\end{table}

\newpage
\subsection{More Experimental Results}
\begin{figure*}[h]
\vskip 0.2in
\begin{center}
\centerline{\includegraphics[width=0.9\textwidth]{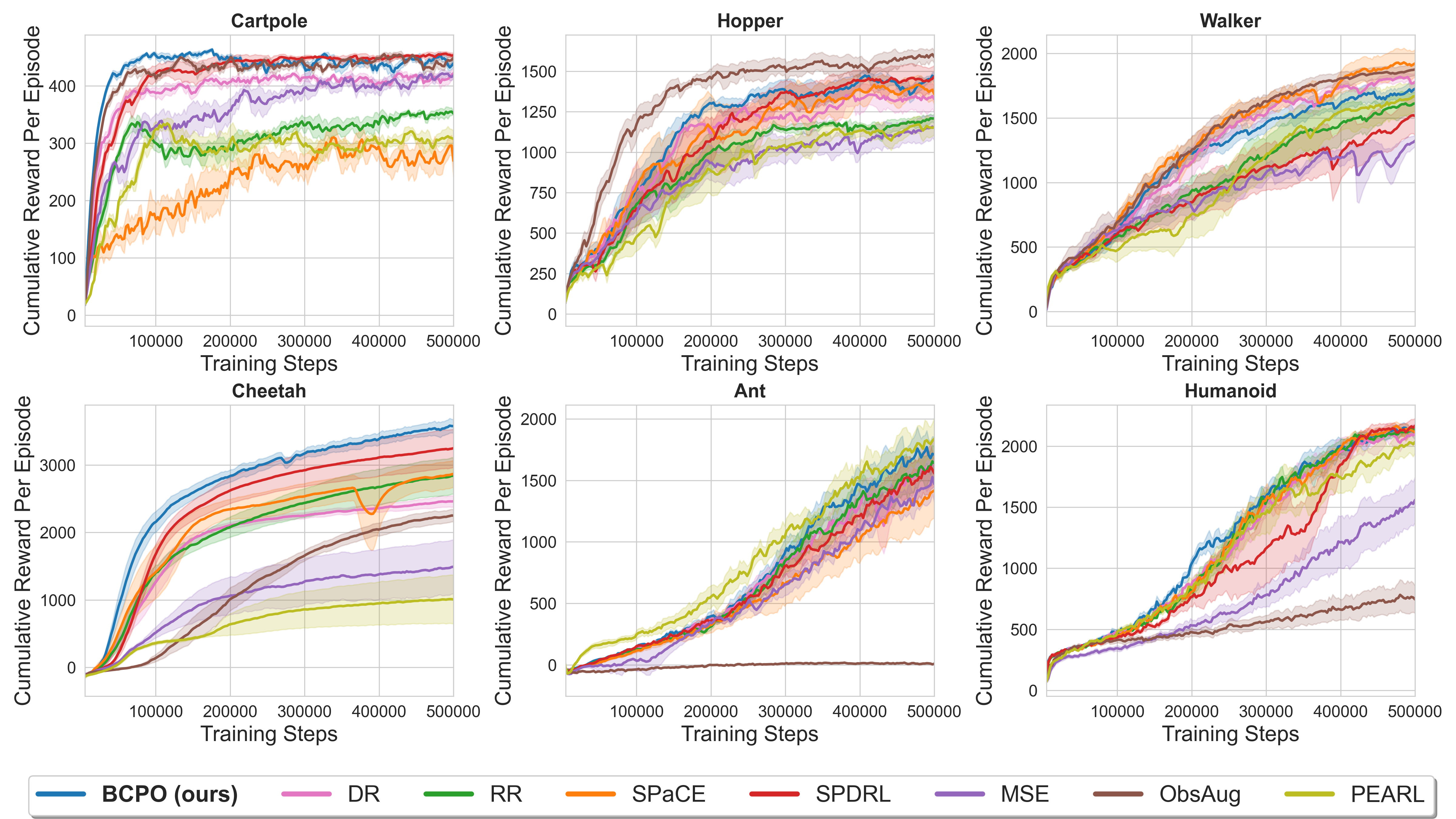}}
\caption{\textbf{Training Performance.} Returns vs. Training steps. All methods are evaluated on training contexts every 5000 training steps. Results averaged over 5 seeds.}
\label{fig:training_result_plot}
\end{center}
\vskip -0.2in
\end{figure*}


\begin{table}[H]
\caption{Training Performance Across Environments and Context Spaces}
\label{tab:training-results}
\centering
\resizebox{\textwidth}{!}{
\begin{tabular}{@{}lccccccl@{\hskip 0.75em}c@{}}
\toprule
\multirow{2}{*}{Method} & \multicolumn{6}{c}{Average Return} & \multicolumn{2}{c}{Training Context} \\
\cmidrule(lr){2-7} \cmidrule(lr){8-9}
 & CartPole & Hopper & Walker & Cheetah & Ant & Humanoid & Cont. & Disc. \\
\midrule
BCPO \textbf{(ours)} & 432.42 $\pm$ 0.29 & 1470.18 $\pm$ 0.14 & 1739.81 $\pm$ 0.16 & 3594.81 $\pm$ 0.09 & 1739.75 $\pm$ 0.41 & 2108.11 $\pm$ 0.16 &  & $\checkmark$ \\
DR & 429.17 $\pm$ 0.13 & 1331.84 $\pm$ 0.23 & 1741.38 $\pm$ 0.24 & 2467.62 $\pm$ 0.08 & 1688.53 $\pm$ 0.42 & 2129.27 $\pm$ 0.18 & $\checkmark$ & \\
RR & 353.18 $\pm$ 0.37 & 1216.75 $\pm$ 0.24 & 1621.16 $\pm$ 0.18 & 2857.3 $\pm$ 0.19 & 1665.98 $\pm$ 0.45 & 2161.46 $\pm$ 0.16 &  & $\checkmark$\\
SPaCE & 260.09 $\pm$ 0.75 & 1359.44 $\pm$ 0.21 & 1928.08 $\pm$ 0.15 & 2878.14 $\pm$ 0.14 & 1497.96 $\pm$ 0.45 & 2102.5 $\pm$ 0.16 &  &$\checkmark$ \\
SPDRL & 451.18 $\pm$ 0.05 & 1484.02 $\pm$ 0.13 & 1529.02 $\pm$ 0.22 & 3251.23 $\pm$ 0.18 & 1576.71 $\pm$ 0.45 & 2200.98 $\pm$ 0.13 & $\checkmark$ & \\
MSE & 410.95 $\pm$ 0.3 & 1167.83 $\pm$ 0.25 & 1346.4 $\pm$ 0.1 & 1505.33 $\pm$ 0.55 & 1525.63 $\pm$ 0.42 & 1601.91 $\pm$ 0.34 & $\checkmark$ & \\
ObsAug & 454.38 $\pm$ 0.18 & 1587.67 $\pm$ 0.16 & 1887.55 $\pm$ 0.05 & 2266.09 $\pm$ 0.09 & 7.93 $\pm$ 6.18 & 726.1 $\pm$ 0.41 & $\checkmark$ & \\
PEARL & 313.22 $\pm$ 0.47 & 1160.73 $\pm$ 0.28 & 1687.12 $\pm$ 0.18 & 1012.33 $\pm$ 0.72 & 1814.71 $\pm$ 0.37 & 2025.95 $\pm$ 0.19 &  & $\checkmark$\\
\bottomrule
\end{tabular}
}
\vspace{0.5em}
\footnotesize
\begin{tabular}{@{}p{\textwidth}@{}}
Training Context: $\checkmark$ under \textit{Cont.} = Continuous parameter training, $\checkmark$ under \textit{Disc.} = Discrete instance training ($|\mathcal{C}|=8$ for CartPole, $4$ for others). Metrics averaged over 5 seeds (mean ± coefficient of variation). N/A indicates failure during training. \\
\end{tabular}
\end{table}

\clearpage
\vskip 0.2in
\bibliography{sample}

\end{document}